\documentclass{article}

\usepackage{microtype}
\usepackage{subfigure}
\usepackage{booktabs}
\usepackage{hyperref}
\usepackage{graphicx}
\usepackage{hyperref}
\usepackage{url}
\usepackage{amsthm}
\usepackage{amssymb}
\usepackage{amsmath}
\usepackage{xcolor}
\usepackage{algorithmic}
\usepackage{mathrsfs}
\usepackage{bbm}
\usepackage{xcolor}
\usepackage{thm-restate}
\usepackage{enumerate}
\usepackage{float}
\usepackage{graphicx}
\usepackage{caption}
\usepackage{trackchanges}

\newtheorem{lemma}{Lemma}

\newtheorem{definition}{Definition}

\newcommand{\mdp}{\mathcal{M}}
\newcommand{\eye}{I}

\newcommand{\states}{\mathcal{S}}
\newcommand{\actions}{\mathcal{A}}
\newcommand{\transitions}{P}
\newcommand{\ppi}{\transitions^\pi}
\newcommand{\rewards}{r}
\newcommand{\valuefunctions}{\mathcal{V}}
\newcommand{\reals}{\mathbb{R}}

\newcommand{\bellop}{\mathcal{T}}
\newcommand{\relint}{\text{relint}_K}
\newcommand{\cl}{\text{cl}_K}
\newcommand{\pdelta}{P_{\delta}}
\renewcommand{\det}{\text{det}}
\newcommand{\adj}{\text{adj}}

\newcommand{\affinesev}{H_{s_1,..,s_k}^{\pi}}

\newcommand{\agree}{Y^{\pi}_{s_1,..,s_k}}

\newcommand{\deter}{D_{s, a}}
\newcommand{\E}{\mathop{\mathbb{E}}}
\newcommand{\policies}{\mathcal{P}(\actions)^\states}
\newcommand{\col}{C^{\pi}}
\newcommand{\colprime}{C^{\pi'}}
\newcommand{\pimuzero}{\pi_{\mu_0}}
\newcommand{\pimuone}{\pi_{\mu_1}}
\newcommand{\pizero}{\pi_{0}}
\newcommand{\pione}{\pi_{1}}
\newcommand{\muzero}{\mu_0}
\newcommand{\muone}{\mu_1}
\newcommand{\lvec}{\preccurlyeq}
\newcommand{\gvec}{\succcurlyeq}
\newcommand{\dotbar}{\cdot \, | \,}

  {\begin{enumerate}%
    \setlength{\itemsep}{2pt}%
    \setlength{\parskip}{0pt}}%
  {\end{enumerate}}

\newcommand{\cbar}{\, | \,}

\def \Ppi {\transitions^\pi}
\def \rpi {\rewards_\pi}
\def \agreeone {Y^\pi_{\states \setminus \{s\}}}

\newcommand{\bR}{\mathbb{R}}
\newcommand{\bN}{\mathbb{N}}
\newcommand{\ssubspace}{\mathscr{S}}

\usepackage[accepted]{icml2019}
\icmltitlerunning{The Value Function Polytope in Reinforcement Learning}

\begin{document}
\twocolumn[
\icmltitle{The Value Function Polytope in Reinforcement Learning}

\icmlsetsymbol{equal}{*}

\begin{icmlauthorlist}
\icmlauthor{Robert Dadashi}{brain}
\icmlauthor{Adrien Ali Ta\"{i}ga}{brain,mila}
\icmlauthor{Nicolas Le Roux}{brain}
\icmlauthor{Dale Schuurmans}{brain,alberta}
\icmlauthor{Marc G. Bellemare}{brain}
\end{icmlauthorlist}

\icmlaffiliation{brain}{Google Brain}
\icmlaffiliation{mila}{Mila, Universit\'{e} de Montr\'{e}al}
\icmlaffiliation{alberta}{Department of Computing Science, University of Alberta}
\icmlcorrespondingauthor{Robert Dadashi}{dadashi@google.com}

\vskip 0.3in
]

\printAffiliationsAndNotice{}

\begin{abstract}
We establish geometric and topological properties of the space of value functions in finite state-action Markov decision processes. 
Our main contribution is the characterization of the nature of its shape: a general polytope \cite{aigner2010proofs}. To demonstrate this result, we exhibit several properties of the structural relationship between policies and value functions including the line theorem, which shows that the value functions of policies constrained on all but one state describe a line segment. Finally, we use this novel perspective to introduce visualizations to enhance the understanding of the dynamics of reinforcement learning algorithms.
\end{abstract}

\section{Introduction}

The notion of value function is central to reinforcement learning (RL). It arises directly in the design of algorithms such as value iteration \citep{Bellman:DynamicProgramming}, policy gradient \cite{sutton2000policy}, policy iteration \citep{howard60dynamic}, and evolutionary strategies \citep[e.g.][]{szita06learning}, which either predict it directly or estimate it from samples, while also seeking to maximize it. The value function is also a useful tool for the analysis of approximation errors \citep{bertsekas96neurodynamic,munos03error}.

In this paper we study the map $\pi \mapsto V^\pi$ from stationary policies, which are typically used to describe the behaviour of RL agents, to their respective value functions. Specifically, we vary $\pi$ over the joint simplex describing all policies and show that the resulting image forms a polytope, albeit one that is possibly self-intersecting and non-convex.

We provide three results all based on the notion of ``policy agreement'', whereby we study the behaviour of the map $\pi \mapsto V^\pi$ as we only allow the policy to vary at a subset of all states.

\noindent \textbf{Line theorem.} We show that policies that agree on all but one state generate a line segment within the value function polytope, and that this segment is monotone (all state values increase or decrease along it).

\noindent \textbf{Relationship between faces and semi-deterministic policies.} We show that $d$-dimensional faces of this polytope are mapped one-to-many to policies which behave deterministically in at least $d$ states.

\noindent \textbf{Sub-polytope characterization.} We use this result to generalize the line theorem to higher dimensions, and demonstrate that varying a policy along $d$ states generates a $d$-dimensional sub-polytope.

Although our ``line theorem'' may not be completely surprising or novel to expert practitioners, we believe we are the first to highlight its existence. In turn, it forms the basis of the other two results, which require additional technical machinery which we develop in this paper, leaning on results from convex analysis and topology.

While our characterization is interesting in and of itself, it also opens up new perspectives on the dynamics of learning algorithms. We use the value polytope to visualize the expected behaviour and pitfalls of common algorithms: value iteration, policy iteration, policy gradient, natural policy gradient \citep{kakade2002natural}, and finally the cross-entropy method \citep{deboer04tutorial}. 

\begin{figure}[h!]
\centering{
\includegraphics[width=0.27\textwidth]{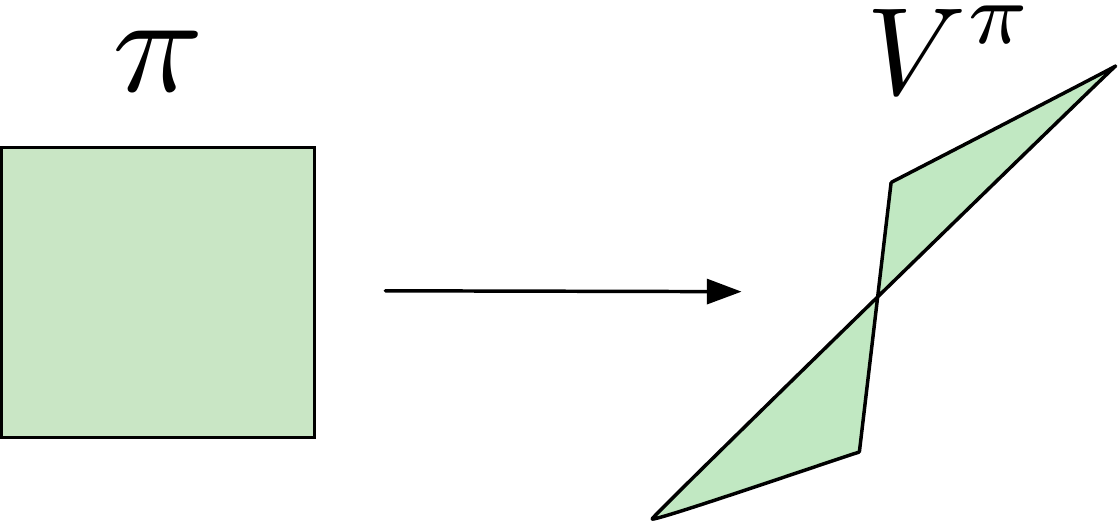}}
\caption{Mapping between policies and value functions.\label{fig:mapping_from_policies_to_value_functions}}
\end{figure}

\section{Preliminaries}

We are in the reinforcement learning setting \cite{sutton2018reinforcement}. We consider a Markov decision process $\mathcal{M} := \langle\states , \actions, \rewards, \transitions, \gamma\rangle$ with $\states$ the finite state space, $\actions$ the finite action space, $\rewards$ the reward function, $\transitions$ the transition function, and $\gamma$ the discount factor for which we assume $\gamma \in [0,1)$.  We denote the number of states by $|\states|$, the number of actions by $|\actions|$.

A stationary policy $\pi$ is a mapping from states to distributions over actions; we denote the space of all policies by $\policies$. Taken with the transition function $P$, a policy defines a state-to-state transition function $P^\pi$:
\begin{equation*}
    P^\pi(s' \cbar s) = \sum_{a \in \actions} \pi(a \cbar s) P(s' \cbar s, a).
\end{equation*}
The value $V^{\pi}$ is defined as the expected cumulative reward from starting in a particular state and acting according to $\pi$:
\begin{equation*}
V^{\pi}(s) = \E\nolimits_{P^\pi}\Big(\sum^{\infty}_{i=0} \gamma^{i}r(s_i, a_i) \cbar s_0=s \Big).
\end{equation*}
The Bellman equation \cite{Bellman:DynamicProgramming} connects the value function $V^{\pi}$ at a state $s$ with the value function at the subsequent states when following $\pi$:
\begin{align} V^\pi(s) = \E\nolimits_{P^\pi}\Big(r(s, a) + \gamma V^\pi(s')\Big)
.
\label{eq:bellman}
\end{align}
Throughout we will make use of vector notation \citep[e.g.][]{puterman94markov}. Specifically, we view (with some abuse of notation) $P^\pi$ as a $|\states| \times |\states|$ matrix, $V^\pi$ as a $|\states|$-dimensional vector, and write $\rewards_\pi$ for the vector of expected rewards under $\pi$. In this notation, the Bellman equation for a policy $\pi$ is
\begin{equation*}
    V^\pi = \rewards_\pi + \gamma \Ppi V^\pi = (I - \gamma \Ppi)^{-1} \rpi .
\end{equation*}

In this work we study how the value function $V^{\pi}$ changes as we continuously vary the policy $\pi$. As such, we will find convenient to also view this value function as the functional
\begin{align*}
  f_v: \; \policies &\rightarrow \reals^{\states}\\
  \pi &\mapsto V^{\pi} = (\eye - \gamma  \Ppi)^{-1}\rpi.
\end{align*}
 We will use the notation $V^\pi$ when the emphasis is on the vector itself, and $f_v$ when the emphasis is on the mapping from policies to value functions.

Finally, we will use $\lvec$ and $\gvec$ for element-wise vector inequalities, and for a function $f : \mathcal{F} \to \mathcal{G}$ and a subset $F \subset \mathcal{F}$ write $f(F)$ to mean the image of $f$ applied to $F$.

\subsection{Polytopes in $\reals^n$}
 
 Central to our work will be the result that the image of the functional $f_v$ applied to the space of policies forms a \emph{polytope}, possibly nonconvex and self-intersecting, with certain structural properties. This section lays down some of the necessary definitions and notations.
 For a complete overview on the topic, we refer the reader to \citet{grunbaum1967convex, ziegler2012lectures, brondsted2012introduction}.

We begin by characterizing what it means for a subset $P \subseteq \reals^n$ to be a convex polytope or polyhedron. In what follows we write $Conv(x_1, \dots, x_k)$ to denote the convex hull of the points $x_1, \dots, x_k$.
\begin{definition}[Convex Polytope]
$P$ is a convex polytope iff there are $k \in \bN$ points $x_1, x_2, ..., x_k \in \reals^n$ such that $P = Conv(x_1, \dots, x_k)$.
\end{definition}

\begin{definition}[Convex Polyhedron]
P is a convex polyhedron iff there are $k \in \bN$ half-spaces $\hat{H}_1, \hat{H}_2, ..., \hat{H}_k$ whose intersection is $P$, that is
\begin{equation*}
    P=\cap_{i=1}^{k} \hat{H}_k.
\end{equation*}
\end{definition}
A celebrated result from convex analysis relates these two definitions: a \emph{bounded}, convex polyhedron is a convex polytope \cite{ziegler2012lectures}.

The next two definitions generalize convex polytopes and polyhedra to non-convex bodies.
\begin{definition}[Polytope]\label{def:general_polytope}
A (possibly non-convex) polytope is a finite union of convex polytopes.
\end{definition}
\begin{definition}[Polyhedron]
A (possibly non-convex) polyhedron is a finite union of convex polyhedra.
\end{definition}
We will make use of another, recursive characterization based on the notion that the boundaries of a polytope should be ``flat'' in a topological sense \cite{klee1959some}.

For an affine subspace $K \subseteq \bR^n$, $V_x \subset K$ is a \emph{relative neighbourhood} of $x$ in $K$ if $x \in V_x$ and $V_x$ is open in $K$. For $P \subset K$, the \emph{relative interior} of $P$ in $K$, denoted $\text{relint}_K(P)$, is then the set of points in $P$ which have a relative neighbourhood in $K \cap P$. The notion of ``open in $K$'' is key here: a point that lies on an edge of the unit square does not have a relative neighbourhood in the square, but it has a relative neighbourhood in that edge. The \emph{relative boundary} $\partial_K P$ is defined as the set of points in $P$ not in the relative interior of $P$, that is
\begin{equation*}
    \partial_K P = P \setminus \text{relint}_K(P) .
\end{equation*}
Finally, we recall that $H \subseteq K$ is a \emph{hyperplane} if $H$ is an affine subspace of $K$ of dimension $\text{dim}(K)-1$.
\begin{restatable}{proposition}{polyhedrabyboundaries}
$P$ is a polyhedron in an affine subspace $K \subseteq \bR^n$ if
\vspace{-4mm}
\begin{enumerate}[(i)]
\item $P$ is closed;
\item There are $k \in \bN$ hyperplanes $H_1,..,H_k$ in $K$ whose union contains the boundary of $P$ in $K$:\\ $\partial_K P \subset \cup^{k}_{i=1} H_i$; and
\item For each of these hyperplanes, $P\cap H_i$ is a polyhedron in $H_i$.
\end{enumerate}
\label{prop:tope}
\end{restatable}

All proofs may be found in the appendix.

\section{The Space of Value Functions \label{sec:vf}}

We now turn to the main object of our study, the \emph{space of value functions} $\valuefunctions$. The space of value functions is the set of all value functions that are attained by some policy. As noted earlier, this corresponds to the image of $\policies$ under the mapping $f_v$:
\begin{align}
\valuefunctions = f_v(\policies) = \Big\{ f_v(\pi) \cbar \pi \in \policies \Big\}.
\end{align}

As a warm-up, Figure \ref{fig:polytope_examples} depicts the space $\valuefunctions$ corresponding to four 2-state MDPs; each set is made of value functions corresponding to 50,000 policies sampled uniformly at random from $\policies$. The specifics of all MDPs depicted in this work can be found in Appendix \ref{sec:mdps}. 

\begin{figure}[h!]
\includegraphics[width=0.45\textwidth]{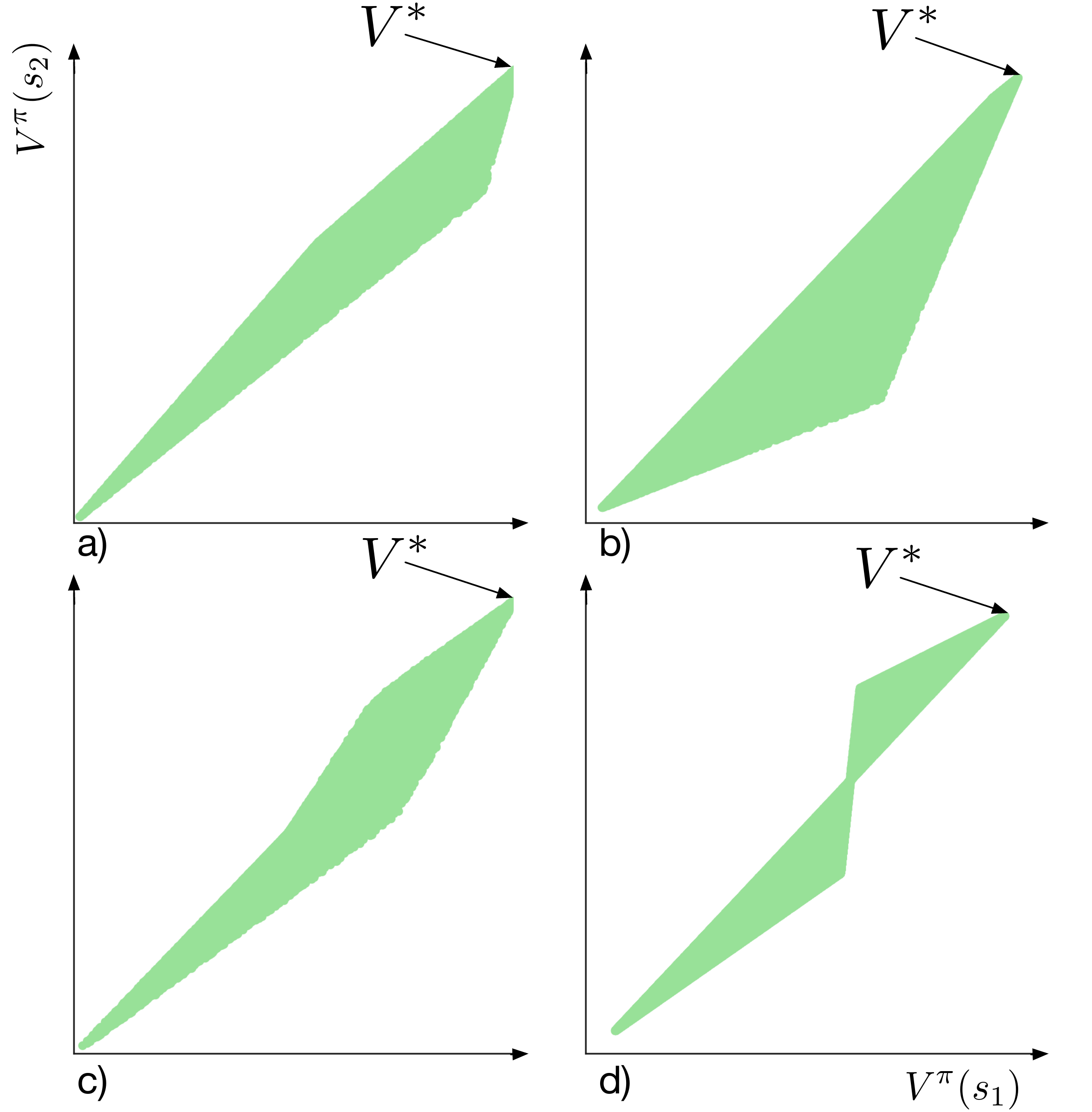}
\caption{Space of value functions for various two-state MDPs.}
\label{fig:polytope_examples}
\end{figure}

While the space of policies $\policies$ is easily described (it is the Cartesian product of $|\states|$ simplices), value function spaces arise as complex polytopes. Of note, they may be non-convex -- justifying our more intricate definition.

In passing, we remark that the polytope gives a clear illustration of the following classic results regarding MDPs \citep[e.g.][]{bertsekas96neurodynamic}:
\begin{itemize}
    \item (Dominance of $V^*$) The optimal value function $V^*$ is the unique dominating vertex of $\valuefunctions$;
    \item (Monotonicity) The edges of $\valuefunctions$ are oriented with the positive orthant;
    \item (Continuity) The space $\valuefunctions$ is connected.
\end{itemize}
The next sections will formalize these and other, less-understood properties of the space of value functions.

\subsection{Basic Shape from Topology}
We begin with a first result on how the functional $f_v$ transforms the space of policies into the space of value functions (Figure \ref{fig:mapping_from_policies_to_value_functions}). Recall that 
\begin{equation*}
    f_v(\pi) = (I - \gamma \Ppi)^{-1} \rpi.
\end{equation*}
Hence $f_v$ is infinitely differentiable everywhere on $\policies$ (Appendix \ref{sec:supp_results}). The following is a topological consequence of this property, along with the fact that $\policies$ is a compact and connected set.
\begin{restatable}{lemma}{lemmacompact}
The space of value functions $\valuefunctions$ is compact and connected.
\label{lm:compact}
\end{restatable}
The interested reader may find more details on this topological argument in \cite{engelking1989general}.

\subsection{Policy Agreement and Policy Determinism}

Two notions play a central role in our analysis:
\textit{policy agreement} and \textit{policy determinism}.

\begin{definition}[Policy Agreement]
Two policies $\pi_1, \pi_2$ \textit{agree} on states $s_1,.., s_k \in \states$ if $\pi_1(\cdot \cbar s_i) = \pi_2(\cdot \cbar s_i)$ for each $s_i$, $i = 1, \dots, k$.
\end{definition}
For a given policy $\pi$, we denote by $Y^\pi_{s_1, \dots, s_k} \subseteq \policies$ the set of policies which agree with $\pi$ on $s_1, \dots, s_k$; we will also write $\agreeone$ to describe the set of policies that agree with $\pi$ on all states except $s$. Note that policy agreement does not imply disagreement; in particular, $\pi \in Y^\pi_{\ssubspace}$ for any subset of states $\ssubspace \subset \states$.
\begin{definition}[Policy Determinism]
A policy $\pi$ is 
\begin{enumerate}[(i)]
    \item $s$-\textit{deterministic} for $s \in \states$ if $\pi(a \cbar s) \in \{0, 1\}$.
    \item \textit{semi-deterministic} if it is $s$-deterministic for at least one $s \in \states$.
    \item \textit{deterministic} if it is $s$-deterministic for all states $s \in \states$.
\end{enumerate}
\end{definition}
We will denote by $\deter$ the set of semi-deterministic policies that take action $a$ when in state $s$.

\begin{restatable}{lemma}{lmzeros}
Consider two policies $\pi_1, \pi_2$ that agree on $s_1, \dots, s_k \in \states$. Then the vector $r_{\pi_1} - r_{\pi_2}$ has zeros in the components corresponding to $s_1, \dots, s_k$ and the matrix $P^{\pi_1} - P^{\pi_2}$ has zeros in the corresponding rows.
\label{lm:zeros}
\end{restatable}

This lemma highlights that when two policies agree on a given state they have the same immediate dynamic on this state, i.e. they get the same expected reward, and have the same next state transition probabilities. Lemma \ref{lm:freedom} in Section \ref{sec:line} will be a direct consequence of this property.

\subsection{Value Functions and Policy Agreement\label{sec:line}}
We begin our characterization by considering the subsets of value functions that are generated when the action probabilities are kept fixed at certain states, that is: when we restrict the functional $f_v$ to the set of policies that agree with some base policy $\pi$ on these states.

Something special arises when we keep the probabilities fixed at all but state $s$: the functional $f_v$ draws a line segment which is oriented in the positive orthant (that is, one end dominates the other end). Furthermore, the extremes of this line segment can be taken to be $s$-deterministic policies. This is the main result of this section, which we now state more formally.
\begin{restatable}{theorem}{linetheorem}[Line Theorem]
Let $s$ be a state and $\pi$, a policy.
Then there are two $s\text{-deterministic}$ policies in $\agreeone$, denoted $\pi_l, \pi_u$, which
bracket the value of all other policies $\pi' \in \agreeone$: \begin{equation*}
    f_v(\pi_l) \lvec f_v(\pi') \lvec f_v(\pi_u).
\end{equation*}
Furthermore, the image of $f_v$ restricted to $\agreeone$ is a line segment, and the following three sets are equivalent:
\begin{enumerate}[(i)]
\item $f_v \big (\agreeone \big )$,
\item $\{ f_v(\alpha \pi_l + (1 - \alpha) \pi_u) \cbar \alpha \in [ 0, 1] \}$,
\item $\{ \alpha f_v(\pi_l) + (1 - \alpha) f_v(\pi_u) \cbar \alpha \in [ 0, 1] \}$ .
\end{enumerate}
\label{th:line}
\end{restatable}
The second part of Theorem \ref{th:line} states that one can generate the set of value functions $f_v(\agreeone)$ in two ways: either by drawing the line segment in value space, $f_v(\pi_l)$ to $f_v(\pi_u)$, or drawing the line segment in policy space, from $\pi_l$ to $\pi_u$ and then mapping to value space. Note that this result is a consequence from the Sherman-Morrison formula, which has been used in reinforcement learning for efficient sequential matrix inverse estimation \cite{bradtke1996linear}. While somewhat technical, this characterization of line segment is needed to prove some of our later results. Figure \ref{fig:line_drawn} illustrates the path drawn by interpolating between two policies that agree on state $s_2$.

\begin{figure}[h]
\includegraphics[width=0.45\textwidth]{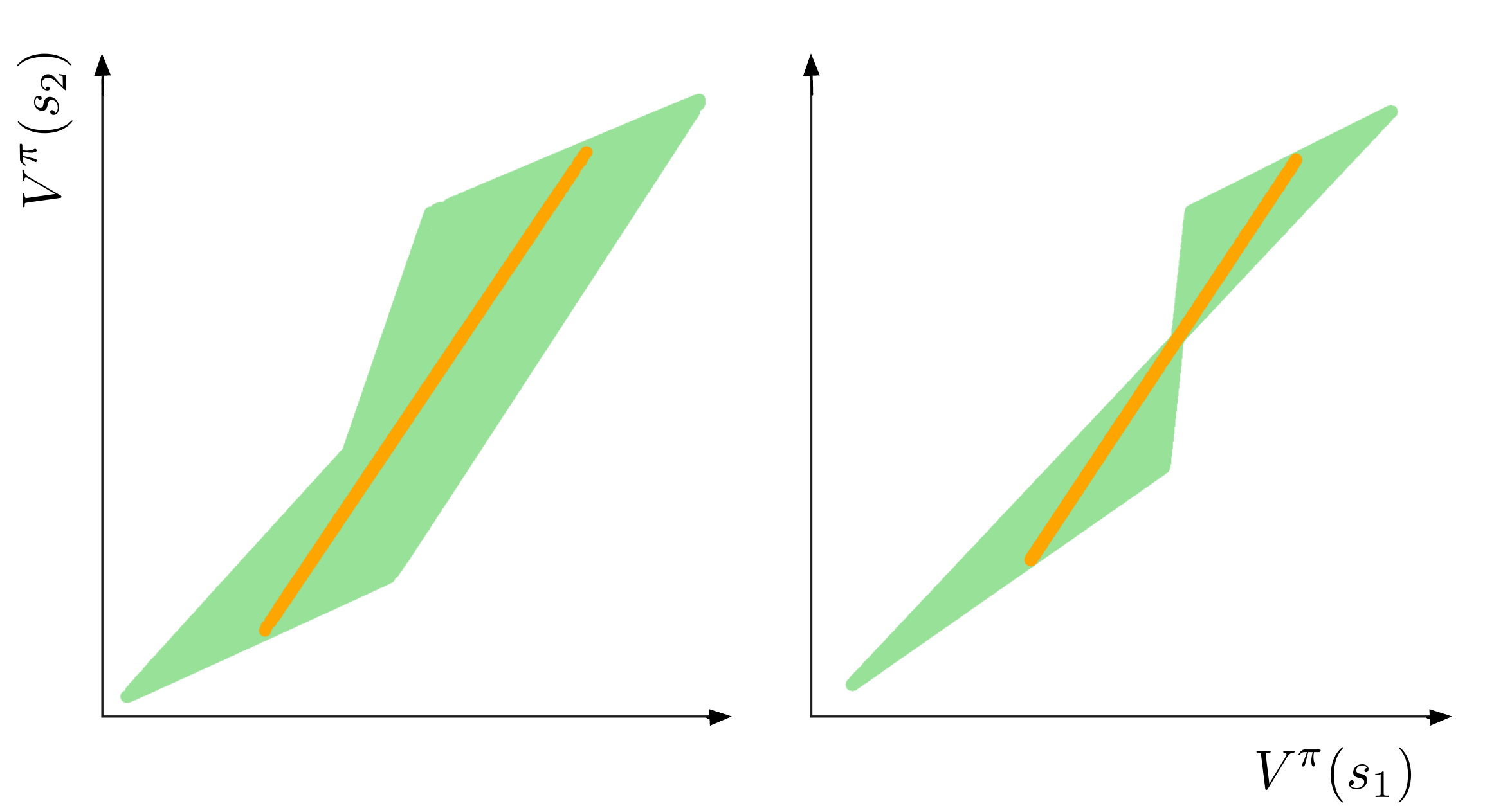}
\caption{Illustration of Theorem \ref{th:line}. The orange points are the value functions of mixtures of policies that agree everywhere but one state.}
\label{fig:line_drawn}
\end{figure}

Theorem \ref{th:line} depends on two lemmas, which we now provide in turn. Consider a policy $\pi$ and $k$ states $s_1, \dots, s_k$, and write $C_{k+1}^\pi, \dots, C_{|\states|}^\pi$ for the columns of the matrix $(I - \gamma \Ppi)^{-1}$ corresponding to states \emph{other} than $s_1, \dots, s_k$. Define the affine vector space
\begin{equation*}
    H^\pi_{s_1, \dots, s_k} = V^\pi + Span(C_{k+1}^\pi, \dots, C_{|\states|}^\pi) .
\end{equation*}
\begin{restatable}{lemma}{linelemmaone}
Consider a policy $\pi$ and $k$ states $s_1, \dots, s_k$. Then the value functions generated by $\agree$ are contained in the affine vector space $\affinesev$:
\begin{align*}
    f_v(\agree) = \valuefunctions \cap \affinesev.
\end{align*}
\label{lm:freedom}
\end{restatable}
Put another way, Lemma \ref{lm:freedom} shows that if we fix the policies on $k$ states, the induced space of value function loses at least $k$ degrees of freedom, specifically that it lies in a $|\states|-k$ dimensional affine vector space.

For $k = |\states|-1$, Lemma \ref{lm:freedom} implies that the value functions lie on a line -- however, the following is necessary to expose the full structure of $\valuefunctions$ within this line.
\begin{restatable}{lemma}{lemmagproperties}
Consider the ensemble $\agreeone$ of policies that agree with a policy $\pi$ everywhere but on $s \in \states$. For $\pi_0, \pi_1 \in \agreeone$ define the function $g:[0, 1]  \rightarrow \valuefunctions$
\begin{equation*}
    g(\mu) = f_v(\mu\pi_1 + (1-\mu) \pi_0).
\end{equation*}
Then the following hold regarding $g$:
\begin{enumerate}[(i)]
\item $g$ is continuously differentiable;
\item (Total order) $g(0) \lvec g(1)$ or $g(0) \gvec g(1)$;
\item If $g(0) = g(1)$  then $g(\mu)=g(0)$, $\mu \in [0, 1]$;
\item (Monotone interpolation) If $g(0) \neq g(1)$ there is a $\rho: [0, 1] \rightarrow \reals$ such that $g(\mu) = \rho(\mu) g(1) + (1 - \rho(\mu)) g(0)$, and $\rho$ is a strictly monotonic rational function of $\mu$.
\end{enumerate}
\label{lm:gproperties}
\end{restatable}

The result (ii) in Lemma \ref{lm:gproperties} was established in \cite{mansour1999complexity} for deterministic policies. Note that in general, $\rho(\mu) \ne \mu$ in the above, as the following example demonstrates.

\begin{restatable}{example}{exnonaffine}
Suppose $\states = \{s_1, s_2\}$, with $s_2$ terminal with no reward associated to it, $\actions=\{a_1, a_2\}$. The transitions and rewards are defined by $P(s_2|s_1, a_2)= 1, P(s_1, |s_1, a_1) = 1, r(s_1, a_1) = 0, r(s_1, a_2)=1$. Define two deterministic policies $\pi_1, \pi_2$ such that $\pi_1(a_1|s_1)=1, \pi_2(a_2|s_1)=1$. We have
$$f_v((1 - \mu)\pi_1 + \mu \pi_2) = \begin{bmatrix} \frac{\mu}{1 - \gamma(1-\mu)} \\ 0 \end{bmatrix}.$$
\label{ex:nonaffine}
\end{restatable}

Remarkably, Theorem \ref{th:line} shows that policies agreeing on all but one state draw line segments irrespective of the size of the action space; this may be of particular interest in the context of continuous action problems. Second, this structure is unique, in the sense that the paths traced by interpolating between two arbitrary policies may be neither linear, nor monotonic (Figure \ref{fig:nonlinear} depicts two examples).

\begin{figure}[h]
\includegraphics[width=0.45\textwidth]{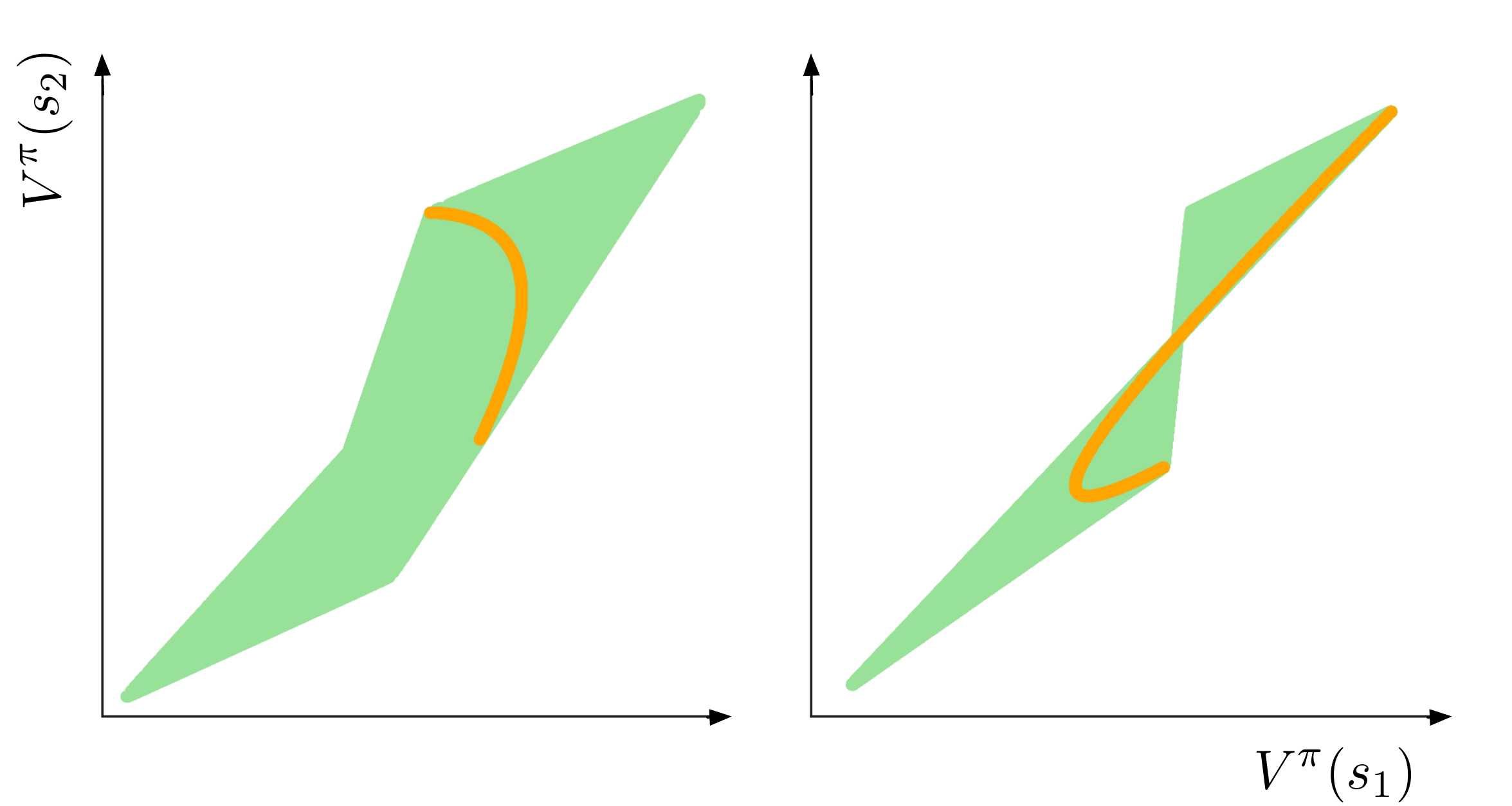}
\caption{Value functions of mixtures of two policies in the general case. The orange points describe the value functions of mixtures of two policies.}
\label{fig:nonlinear}
\end{figure}

\subsection{Convex Consequences of Theorem \ref{th:line}} 

Some consequences arise immediately from Theorem \ref{th:line}. First, the result suggests a recursive application from the value function $V^\pi$ of a policy $\pi$ into its \textit{deterministic constituents}.
\begin{restatable}{corollary}{thmconvexhull}
For any set of states $s_1,.., s_k \in \states$ and a policy $\pi$, $V^{\pi}$ can be expressed as a convex combination of value functions of $\{s_1,.., s_k\}$-deterministic policies. In particular,  $\valuefunctions$ is included in the convex hull of the value functions of deterministic policies.
\label{corol:hull}
\end{restatable}

This result indicates a relationship between the vertices of $\valuefunctions$ and deterministic policies. Nevertheless, we observe in Figure \ref{fig:hull} that the value functions of deterministic policies are not necessarily the vertices of $\valuefunctions$ and that the vertices of $\valuefunctions$ are not necessarily attained by value functions of deterministic policies.

\begin{figure}[h!]
\includegraphics[width=0.45\textwidth]{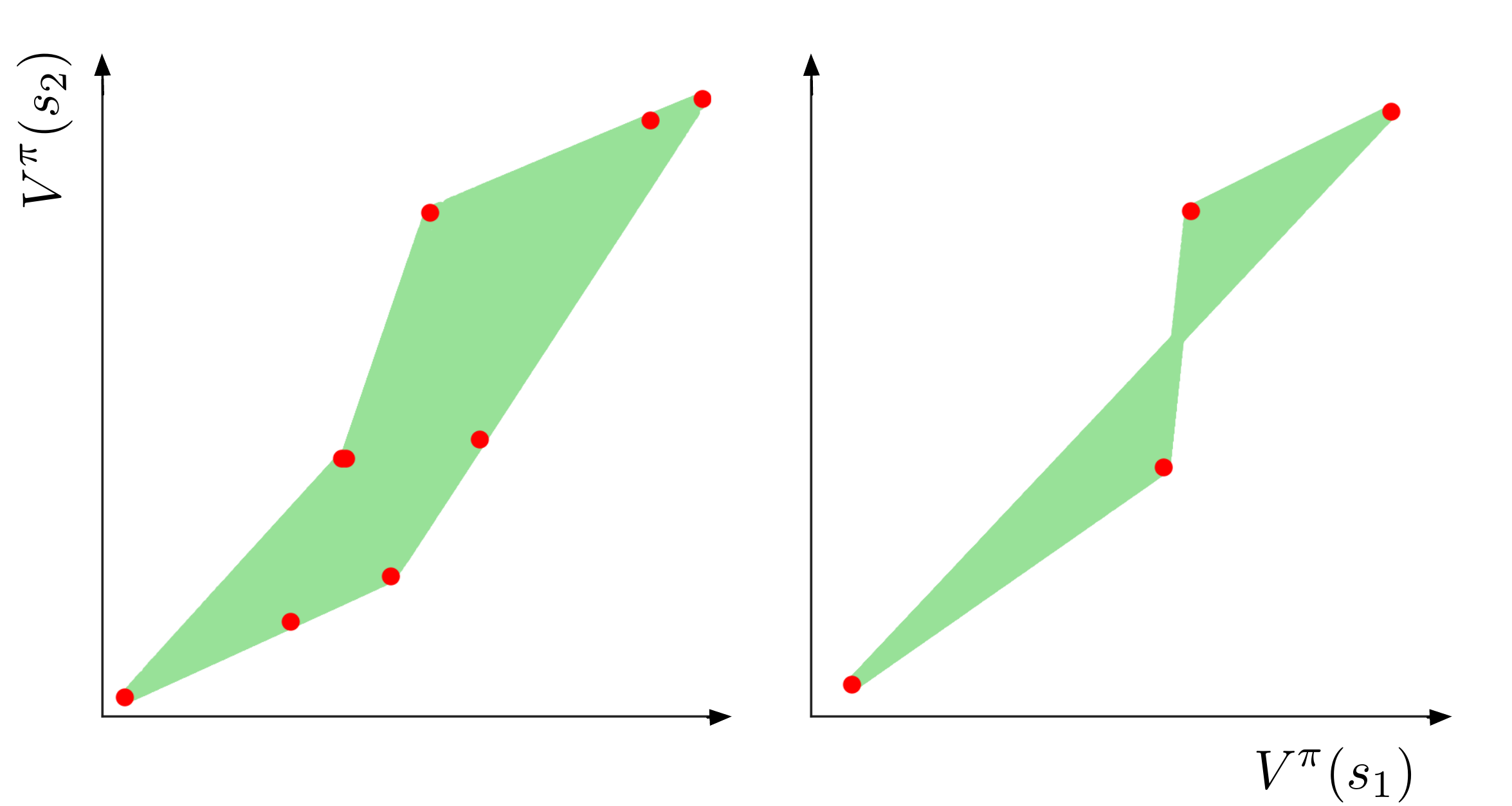}
\caption{Visual representation of Corollary \ref{corol:hull}. The space of value functions is included in the convex hull of value functions of deterministic policies (red dots).}
\label{fig:hull}
\end{figure}

The space of value functions is in general not convex. However, it does possess a weaker structural property regarding paths between value functions which is reminiscent of policy iteration-type results. 
\begin{restatable}{corollary}{nconnect}
Let $V^{\pi}$ and $V^{\pi'}$ be two value functions. Then there exists a sequence of $k \le |\states|$ policies, $\pi_1, \dots, \pi_k$, such that $V^\pi = V^{\pi_1}$, $V^{\pi'} = V^{\pi_k}$, and for every $i \in 1, \dots, k - 1$, the set
\begin{equation*}
\{ f_v( \alpha \pi_i + (1 - \alpha) \pi_{i+1}) \cbar \alpha \in [0, 1] \}
\end{equation*}
forms a line segment.
\label{th:nconnect}
\end{restatable}

\subsection{The Boundary of $\valuefunctions$}
We are almost ready to show that $\valuefunctions$ is a polytope. To do so, however, we need to show that the boundary of the space of value functions is described by semi-deterministic policies.

While at first glance reasonable given our earlier topological analysis, the result is complicated by the many-to-one mapping from policies to value functions, and requires additional tooling not provided by the line theorem. Recall from Lemma \ref{lm:freedom} the use of the affine vector space $\affinesev$ to constrain the value functions generated by fixing certain action probabilities.
\begin{restatable}{theorem}{thneighborhood}
Consider the ensemble of policies $\agree$ that agree with $\pi$ on states $\ssubspace = \{ s_1,..,s_k \}$. Suppose $\forall s \notin \ssubspace$, $\forall a \in \actions$, $\nexists \pi' \in \agree \cap D_{a,s}$ s.t. $f_v(\pi')=f_v(\pi)$, then $f_v(\pi)$ has a relative neighborhood in $\valuefunctions \cap \affinesev$.
\label{th:neighbor}
\end{restatable}

Theorem \ref{th:neighbor} demonstrates by contraposition that the boundary of the space of value functions is a subset of the ensemble of value functions of semi-deterministic policies. Figure \ref{fig:boundary} shows that the latter can be a proper subset.

\begin{restatable}{corollary}{corolboundary}
Consider a policy $\pi \in \policies$,  the states $\ssubspace = \{s_1, .., s_k\}$, and the ensemble  $\agree$ of policies that agree with $\pi$ on $s_1,..,s_k$. Define $\valuefunctions^{y} = f_v(\agree)$, we have that the relative boundary of $\valuefunctions^{y}$ in $\affinesev$ is included in the value functions spanned by policies in $\agree$ that are $s$-deterministic for $s \notin \ssubspace$:
\begin{equation*}
    \partial \valuefunctions^y \subset \bigcup_{s \notin \ssubspace} \bigcup_{a \in \actions} f_v(\agree \cap D_{s, a}),
\end{equation*}
where $\partial$ refers to $\partial_{\affinesev}$.
\label{cr:boundary}
\end{restatable}

\begin{figure}[h!]
\includegraphics[width=0.45\textwidth]{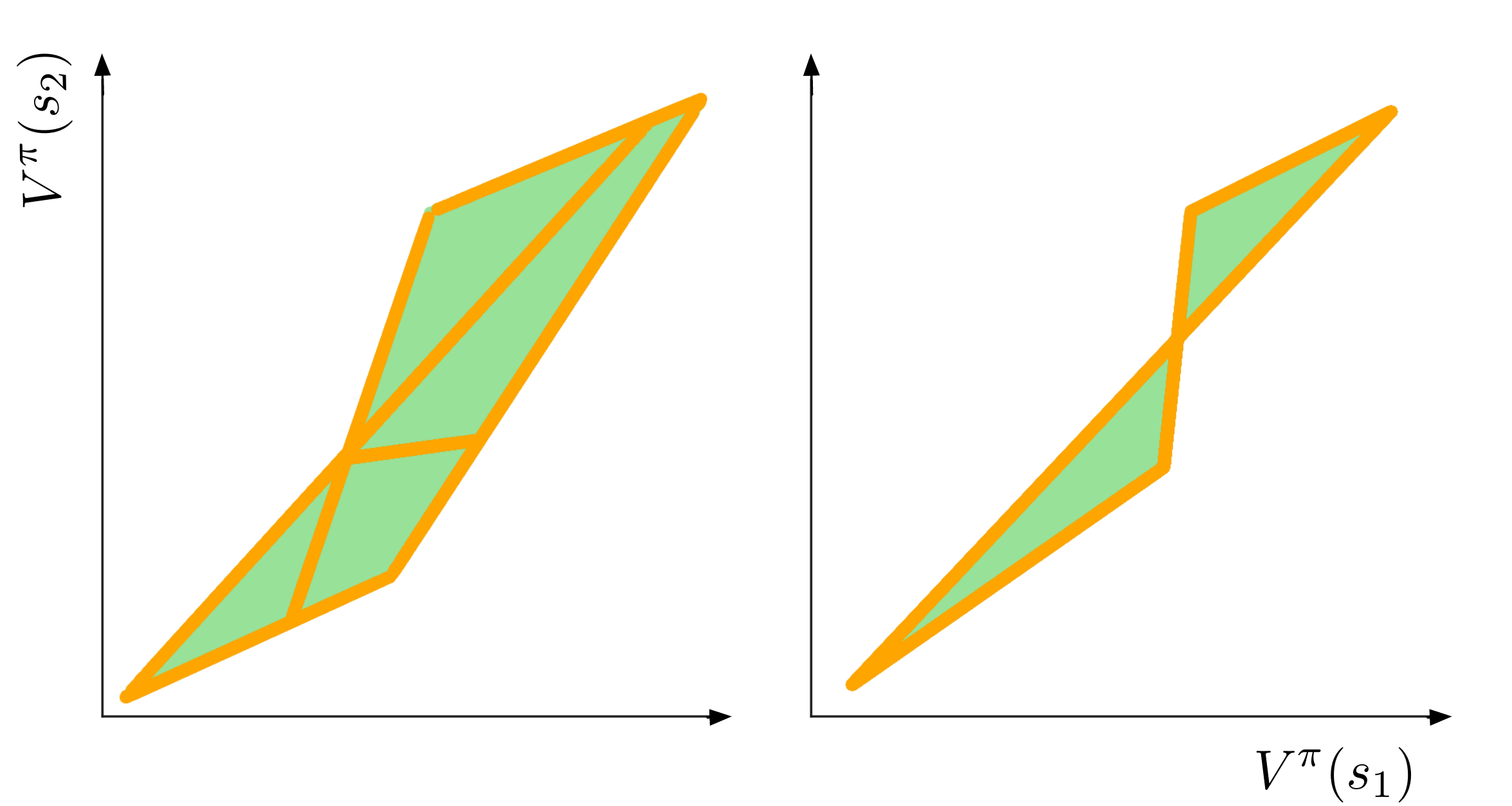}
\caption{Visual representation of Corollary \ref{cr:boundary}. The orange points are the value functions of semi-deterministic policies.}
\label{fig:boundary}
\end{figure}

\subsection{The Polytope of Value Functions}

We are now in a position to combine the results of the previous section to arrive at our main contribution: $\valuefunctions$ is a polytope in the sense of Def. \ref{def:general_polytope} and Prop. \ref{prop:tope}. Our result is in fact stronger: we show that any subset of policies $\agree$ generates a sub-polytope of $\valuefunctions$.
\begin{restatable}{theorem}{thboundaries}
Consider a policy $\pi \in \policies$,  the states $s_1, .., s_k \in \states$, and the ensemble $\agree$ of policies that agree with $\pi$ on $s_1,..,s_k$. Then $f_v(\agree)$ is a polytope and in particular, $\valuefunctions = f_v(Y^\pi_{\emptyset})$ is a polytope.
\label{th:polytope}
\end{restatable}
 Despite the evidence gathered in the previous section in favour of the above theorem, the result is surprising given the fundamental non-linearity of the functional $f_v$: again, mixtures of policies can describe curves (Figure \ref{fig:nonlinear}), and even the mapping $g$ in Lemma \ref{lm:gproperties} is nonlinear in $\mu$.

That the polytope can be non-convex is obvious from the preceding figures. As Figure \ref{fig:boundary} (right) shows, this can happen when value functions along two different line segments cross. At that intersection, something interesting occurs: there are two policies with the same value function but that do not agree on either state. We will illustrate the effect of this structure on learning dynamics in Section \ref{sec:dynamics}.

Finally, there is a natural sub-polytope structure in the space of value functions. If policies are free to vary only on a subset of states of cardinal $k$, then there is a polytope of dimension $k$ associated with the induced space of value functions. This makes sense since constraining policies on a subset of states is equivalent to defining a new MDP, where the transitions associated with the complement of this subset of states are not dependent on policy decisions. 

\section{Related Work}

The link between geometry and reinforcement learning has been so far fairly limited. However we note the former use of convex polyhedra in the following:   

\textbf{Simplex Method and Policy Iteration}. The policy iteration algorithm \cite{howard60dynamic} closely relates to the simplex algorithm \cite{dantzig1948programming}. In fact, when the number of states where the policy can be updated is at most one, it is exactly the simplex method, sometimes referred to as \textit{simple} policy iteration. As opposed to the limitations of the simplex algorithm \cite{littman1995complexity}, namely the worst case convergence in exponential time, it was demonstrated that the simplex algorithm applied to MDPs with an adequate pivot rule converges in polynomial time \cite{ye2011simplex}.

\textbf{Linear Programming}. Finding the optimal value function of an MDP can be formulated as a linear program \cite{puterman94markov, bertsekas96neurodynamic, de2003linear, wang2007dual}. In the primal form, the feasible constraints are defined by $\{ V \in \reals^{|\states|} \; \big| \; V \gvec \bellop^*  V \}$, where $\bellop^*$ is the optimality Bellman operator. Notice that there is a unique value function $V \in \valuefunctions$ that is feasible, which is exactly the optimal value function $V^*$. 

The dual formulation consists of maximizing the expected return for a given initial state distribution, as a function of the discounted state action visit frequency distribution. Contrary to the primal form, any feasible discounted state action visit frequency distribution maps to an \textit{actual} policy  \cite{wang2007dual}. 

\section{Dynamics in the Polytope \label{sec:dynamics}}

In this section we study how the behaviour of common reinforcement learning algorithms is reflected in the value function polytope. We consider two value-based methods, value iteration and policy iteration, three variants of the policy gradient method, and an evolutionary strategy.

Our experiments use the two-state, two-action MDP depicted elsewhere in this paper (details in Appendix \ref{sec:mdps}). Value-based methods are parametrized directly in terms of the value vector in $\reals^2$; policy-based methods are parametrized using the softmax distribution, with one parameter per state. We initialize all methods at the same starting value functions (indicated on Figure \ref{fig:vi}): \textit{near} a vertex ($V^i_1$), \textit{near} a boundary ($V^i_2$), and in the interior of the polytope ($V^i_3$).\footnote{The use of the softmax precludes initializing policy-based methods exactly at boundaries.}

We are chiefly interested in three aspects of the different algorithms' learning dynamics: 1) the \emph{path} taken through the value polytope, 2) the \emph{speed} at which they traverse the polytope, and 3) any \emph{accumulation} points that occur along this path. As such, we compute model-based versions of all relevant updates; in the case of evolutionary strategies, we use large population sizes \cite{deboer04tutorial}.

\subsection{Value Iteration}
Value iteration \cite{Bellman:DynamicProgramming} consists of the repeated application of the optimality Bellman operator $\bellop^*$
$$V_{k+1} := \bellop^* V_{k}.$$
In all cases, $V_0$ is initialized to the relevant starting value function. 
\begin{figure}[h!]
\includegraphics[width=0.475\textwidth]{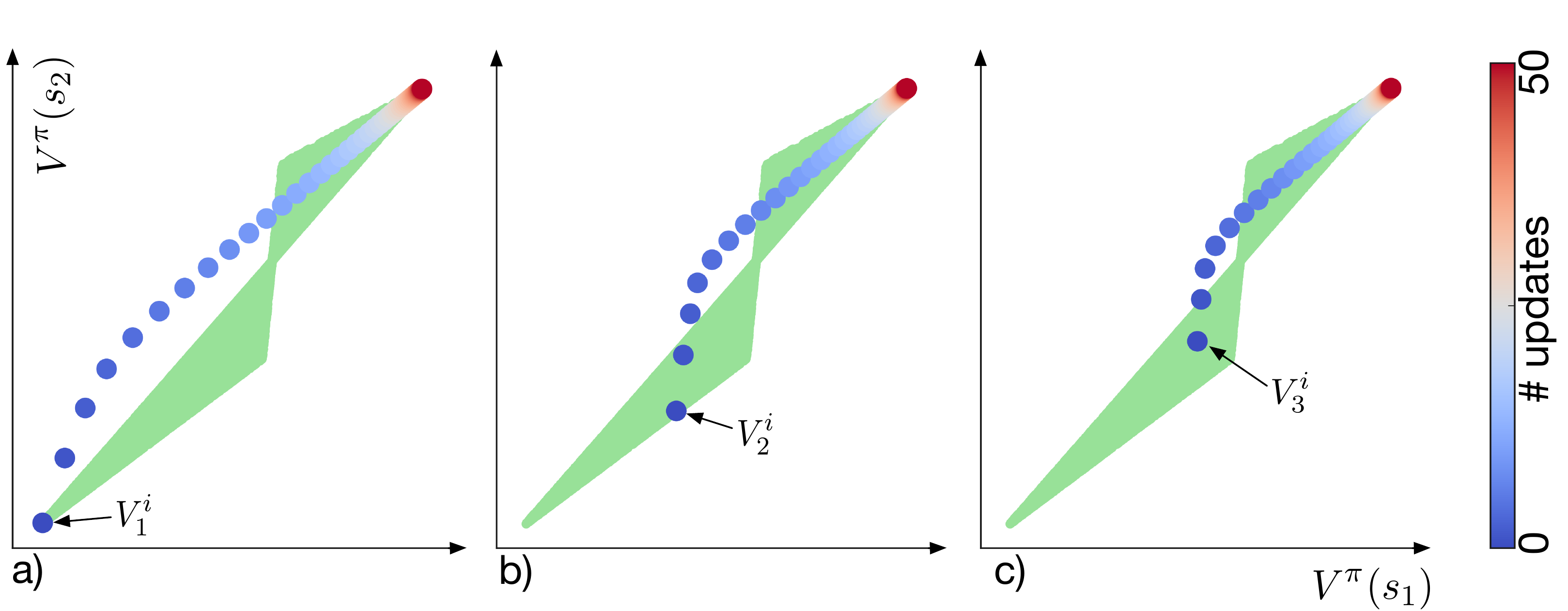}
\caption{Value iteration dynamics for three initialization points.}
\label{fig:vi}
\end{figure}
Figure \ref{fig:vi} depicts the paths in value space taken by value iteration, from the starting point to the optimal value function. We observe that the path does not remain within the polytope: value iteration generates a sequence of vectors that may not map to any policy. Our visualization also highlights results by \citep{bertsekas94generic} showing that value iteration spends most of its time along the constant (1, 1) vector, and that the ``real'' convergence rate is in terms of the second largest eigenvalue of $P$.

\subsection{Policy Iteration}
Policy iteration \cite{howard60dynamic} consists of the repeated application of a policy improvement step and a policy evaluation step until convergence to the optimal policy. The policy improvement step updates the policy by acting \textit{greedily} according to the current value function; the value function of the new policy is then evaluated. The algorithm is based on the following update rule
\begin{align*}
&\pi_{k+1} := \text{greedy}(V_{k})\\
&V_{k+1} := \text{evaluate}(\pi_{k+1}),
\end{align*}
with $V_0$ initialized as in value iteration.
\begin{figure}[h!]
\includegraphics[width=0.5\textwidth]{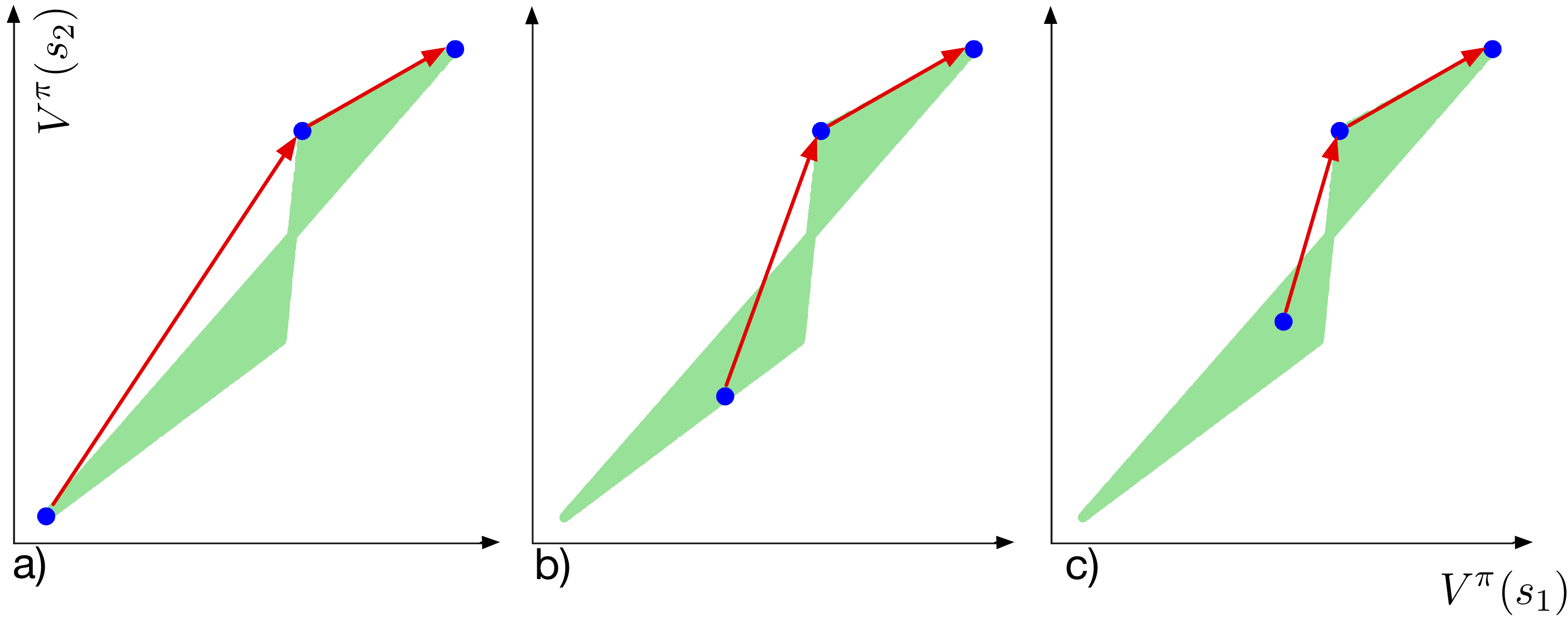}
\caption{Policy iteration. The red arrows show the sequence of value functions (blue) generated by the algorithm.}
\label{fig:pi}
\end{figure}

The sequence of value functions visited by policy iteration (Figure \ref{fig:pi}) corresponds to value functions of  deterministic policies, which in this specific MDP corresponds to vertices of the polytope.

\subsection{Policy Gradient \label{sec:pg}}

Policy gradient is a popular approach for directly optimizing the value function via parametrized policies \cite{williams1992simple,konda2000actor,sutton2000policy}. For a policy $\pi_\theta$ with parameters $\theta$ the policy gradient is
\begin{equation*}
\nabla_\theta J(\theta) = \mathbb{E}_{s \sim d_\pi, a \sim \pi(\cdot \cbar s)} \nabla_\theta \log \pi(a \cbar s) [r(s,a) + \gamma \E V(s')]
\end{equation*}
where $d_\pi$ is the discounted stationary distribution; here we assume a uniformly random initial distribution over the states. The policy gradient update is then ($\eta \in [0, 1])$
\begin{align*}
&\theta_{k+1} := \theta_{k} + \eta \nabla_\theta J(\theta_k) .
\end{align*}

\begin{figure}[h!]
\includegraphics[width=0.5\textwidth]{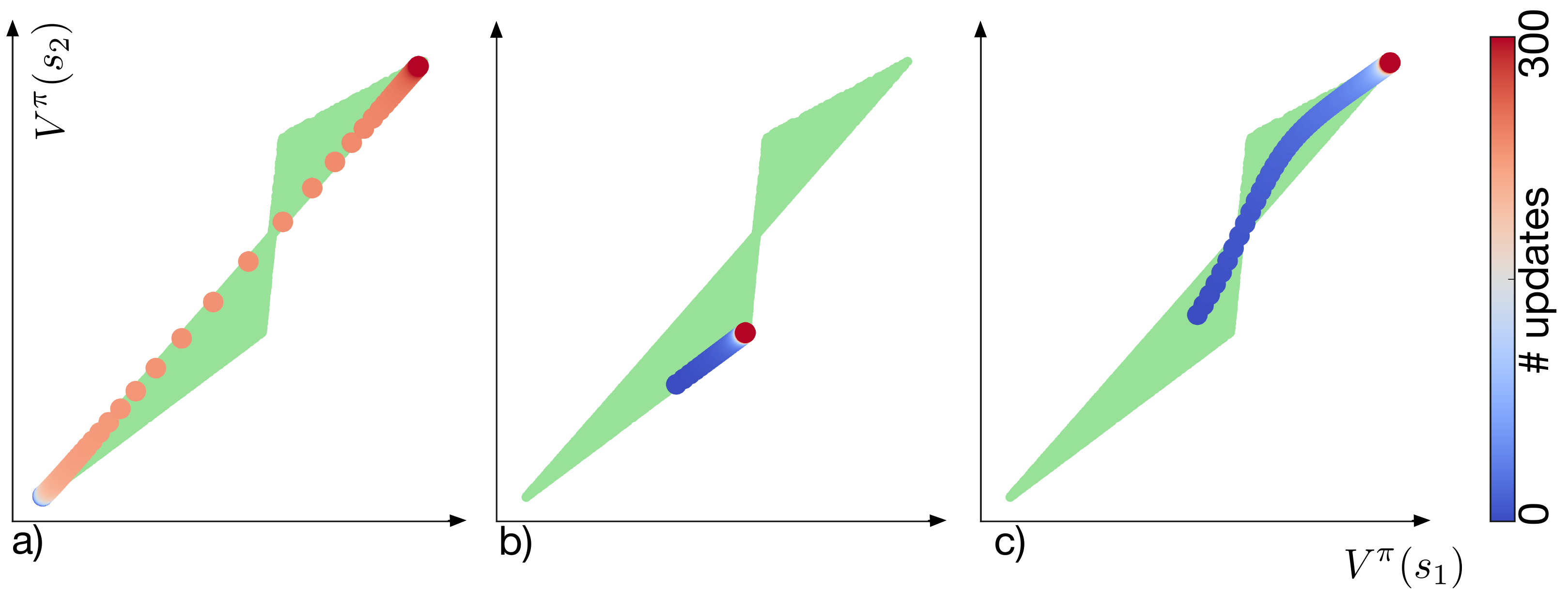}
\caption{Value functions generated by policy gradient.}
\label{fig:pg}
\end{figure}

Figure \ref{fig:pg} shows that the convergence rate of policy gradient strongly depends on the initial condition. In particular, Figure \ref{fig:pg}a),b) show accumulation points along the update path (not shown here, the method does eventually converge to $V^*$).
This behaviour is sensible given the dependence of $\nabla_\theta J(\theta)$ on $\pi(\cdot \cbar s)$, with gradients vanishing at the boundary of the polytope. 

\subsection{Entropy Regularized Policy Gradient}
Entropy regularization adds an entropy term to the objective \cite{williams1991function}. The new policy gradient becomes
\begin{equation*}
    \nabla_\theta J_\text{ent}(\theta) = \nabla_\theta J(\theta) - \nabla_\theta \E_{s \sim d_{\pi}} H(\pi(\cdot \cbar s)),
\end{equation*}
where $H(\cdot)$ denotes the Shannon entropy.
\begin{figure}[h!]
\includegraphics[width=0.5\textwidth]{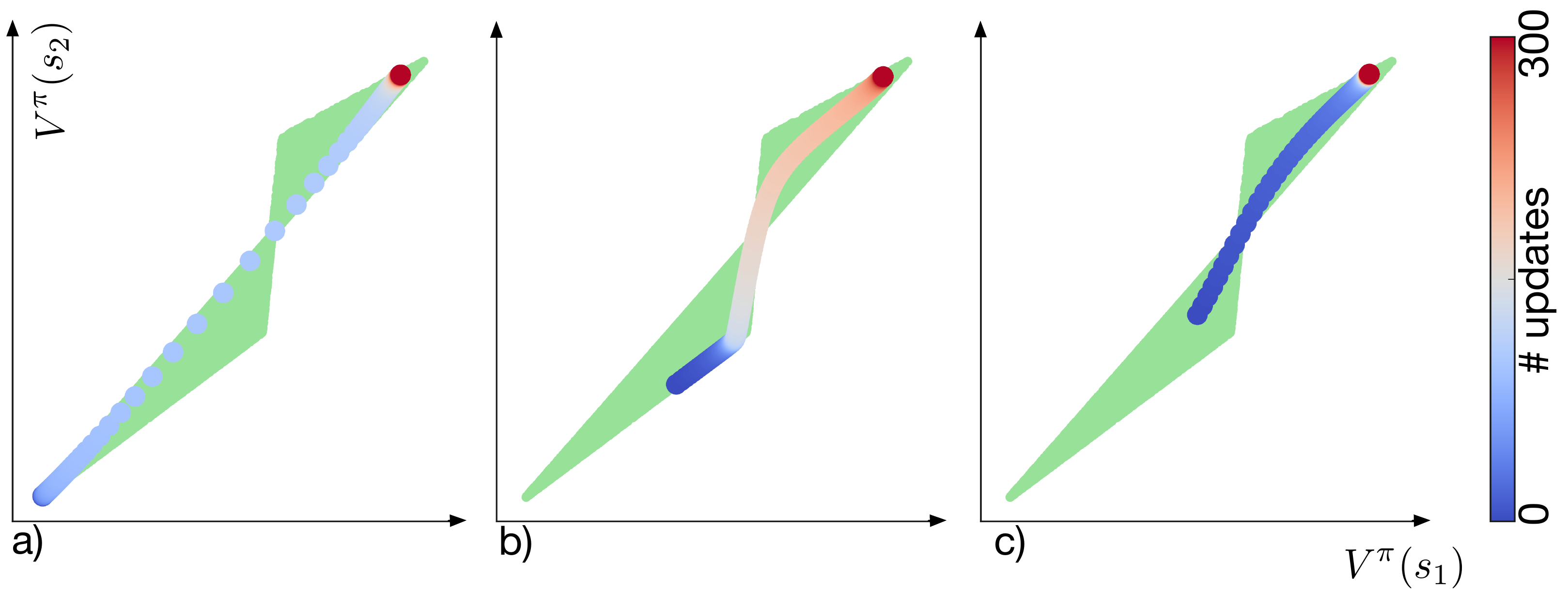}
\caption{Value functions generated by policy gradient with entropy, for three different initialization points.}
\label{fig:pgreg}
\end{figure}
The entropy term encourages policies to move away from the boundary of the polytope. Consequent with our previous observation regarding policy gradient, we find that this improves the convergence rate of the optimization procedure (Figure \ref{fig:pgreg}). One trade-off is that the policy converges to a sub-optimal policy, which is not deterministic. 

\subsection{Natural Policy Gradient}

Natural policy gradient \cite{kakade2002natural} is a second-order policy optimization method. The gradient updates condition the standard policy gradient with the inverse Fisher information matrix $F$ \citep{kakade2002natural}, leading to the following update rule:
\begin{align*}
&\theta_{k+1} := \theta_{k} + \eta F^{-1} \nabla_\theta J(\theta_k).
\end{align*}
This causes the gradient steps to follow the steepest ascent direction in the underlying structure of the parameter space. 

\begin{figure}[h!]
\includegraphics[width=0.475\textwidth]{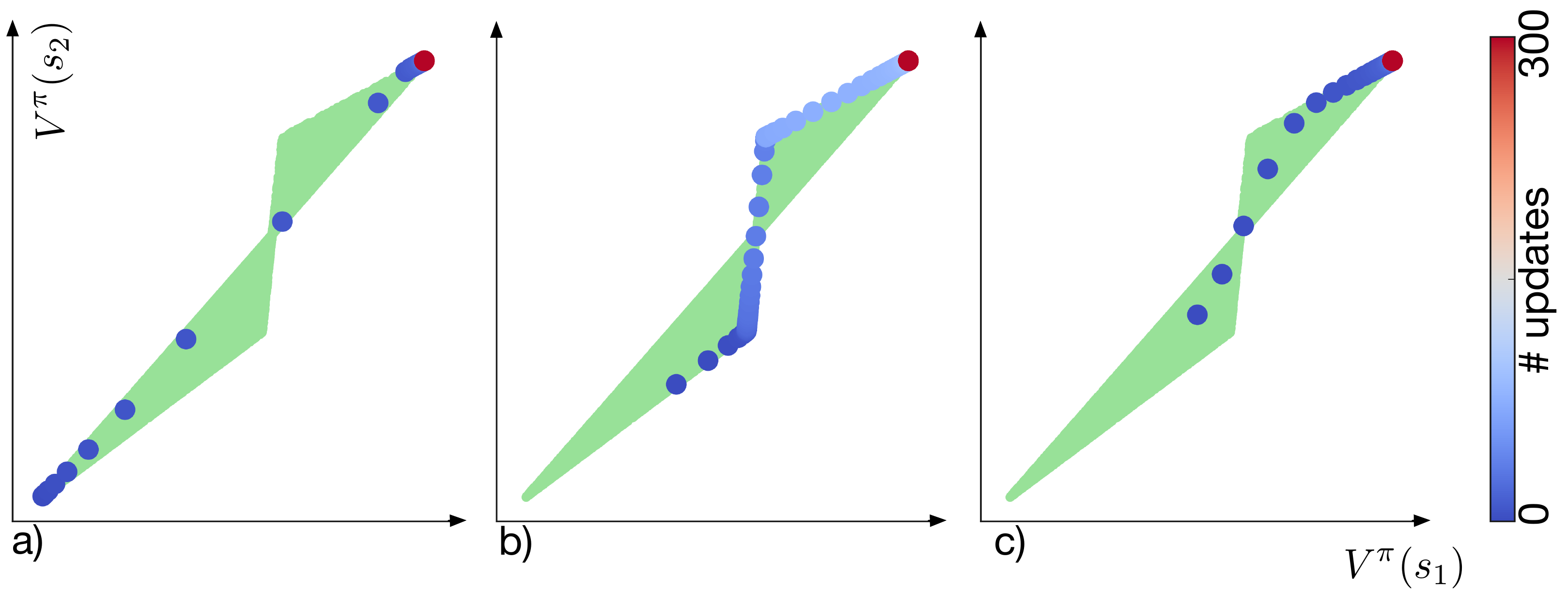}
\caption{Natural policy gradient.}
\label{fig:npg}
\end{figure}
In our experiment, we observe that natural policy gradient is less prone to accumulation than policy gradient (Fig. \ref{fig:npg}), in part because the step-size is better conditioned. Figure b) shows unregularized policy gradient does not, surprisingly enough, take the ``shortest path'' through the polytope to the optimal value function: instead, it moves from one vertex to the next, similar to policy iteration.

\subsection{Cross-Entropy Method}

Gradient-free optimization methods have shown impressive performance over complex control tasks \cite{deboer04tutorial, salimans2017evolution}. We present the dynamics of the cross-entropy method (CEM), without noise and with a constant noise factor (CEM-CN) \cite{szita06learning}. The mechanics of the algorithm is threefold: (i) sample a population of size $N$ of policy parameters from a Gaussian distribution of mean $\theta$, covariance $C$; (ii) evaluate the returns of the population; (iii) select top $K$ members, and fit a new Gaussian onto them. In the CEM-CN variant, we inject additional isotropic noise at each iteration. We use $N=500$, $K=50$, an initial covariance of $0.1\eye$, where $I$ is the identity matrix of size 2, and a constant noise of $0.05\eye$. 

\begin{figure}[h!]
\includegraphics[width=0.475\textwidth]{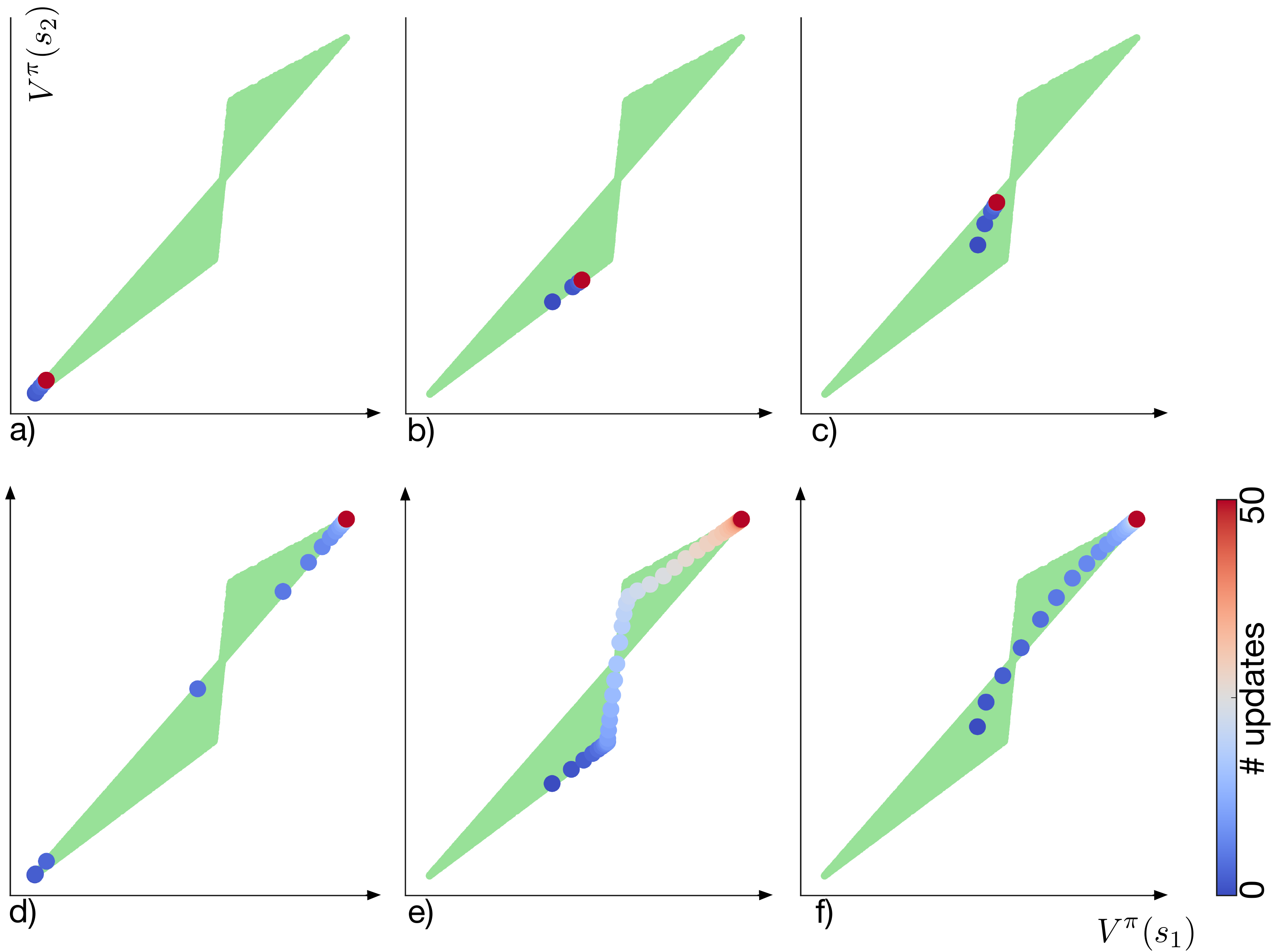}
\caption{The cross-entropy method without noise (CEM) (a, b, c); with constant noise (CEM-CN) (d, e, f).}
\label{fig:cem}
\end{figure}

As observed in the original work \cite{szita06learning}, the covariance of CEM without noise collapses (Figure \ref{fig:cem}.a)b)c)), and therefore reaches convergence for a suboptimal policy. However, the noise addition at each iteration prevents this undesirable behaviour (Figure \ref{fig:cem}.d)e)f)), as the algorithm converges to the optimal value functions for all three initialization points.

\section{Discussion and Concluding Remarks}

In this work, we characterized the shape of value functions and established its surprising geometric nature: a possibly non-convex polytope. This result was based on the line theorem which provides guarantees of monotonic improvement as well as a line-like variation in the space of value functions.  This structural property raises the question of new learning algorithms based on a single state change, and what this might mean in the context of function approximation. 

We noticed the existence of self-intersecting spaces of value functions, which have a bottleneck. However, from our simple study of learning dynamics over a class of reinforcement learning methods, it does not seem that this bottleneck leads to any particular learning slowdown. 

Some questions remain open. Although those geometric concepts make sense for finite state action spaces, it is not clear how they generalize to the continuous case. There is a connection between representation learning and the polytopal structure of value functions that we have started exploring \cite{bellemare2019geometric}. Another exciting research direction is the relationship between the geometry of value functions and function approximation.

\section{Acknowledgements}
The authors would like to thank their colleagues at Google Brain for their help; Carles Gelada, Doina Precup, Georg Ostrovski, Marco Cuturi, Marek Petrik, Matthieu Geist, Olivier Pietquin, Pablo Samuel Castro, R\'{e}mi Munos, R\'{e}mi Tachet, Saurabh Kumar, and Zafarali Ahmed for useful discussion and feedback; Jake Levinson and Mathieu Guay-Paquet for their insights on the proof of Proposition 1; Mark Rowland for providing invaluable feedback on two earlier versions of this manuscript. 

\bibliography{main}
\bibliographystyle{icml2019}

\appendix
\onecolumn
{\Large \bf Appendix}

\section{Details of Markov Decision Processes \label{sec:mdps}}

In this section we give the specifics of the Markov Decision Processes presented in this work. We will use the following convention:
\begin{align*}
&r(s_i, a_j) = \hat{r}[i \times |\actions| + j]\\
&P(s_k | s_i, a_j) = \hat{P}[i \times |\actions| + j][k]
\end{align*}
where $\hat{P}, \hat{r}$ are the vectors given below.\\
\begin{align*}
\text{In Section \ref{sec:vf}, Figure \ref{fig:polytope_examples}: } (a) \; \; &|\actions| = 2, \; \gamma = 0.9 \\
&\hat{r} = [0.06, 0.38, -0.13, 0.64] \\
&\hat{P} = [[ 0.01,  0.99],
            [ 0.92,  0.08],
            [ 0.08, 0.92],
            [ 0.70,  0.30]]\\
\\
(b) \; \; &|\actions| = 2, \;\gamma = 0.9 \\
&\hat{r} = [ 0.88, -0.02, -0.98,  0.42] \\
&\hat{P} = [[ 0.96,  0.04],
       [ 0.19,  0.81],
       [ 0.43,  0.57],
       [ 0.72,  0.28]])\\
\\
(c) \; \;  &|\actions| = 3, \; \gamma = 0.9 \\
&\hat{r} = [-0.93, -0.49,  0.63,  0.78,  0.14, 0.41] \\
&\hat{P} = [[ 0.52,  0.48],
       [ 0.5,  0.5],
       [ 0.99,  0.01],
       [ 0.85 ,  0.15],
       [ 0.11,  0.89],
       [ 0.1,  0.9]]\\
\\
(d) \; \; &|\actions| = 2, \; \gamma = 0.9 \\
&\hat{r} = [-0.45, -0.1,  0.5,  0.5]\\
&\hat{P} = [[ 0.7,  0.3],
       [ 0.99,  0.01],
       [ 0.2,  0.8],
       [ 0.99,  0.01]]\\
\\
\text{In Section \ref{sec:vf}, Figure \ref{fig:line_drawn}, \ref{fig:nonlinear}, \ref{fig:hull}, \ref{fig:boundary}: (left)} \; \; &|\actions| = 3, \; \gamma = 0.8 \\
&\hat{r} = [-0.1, -1.,  0.1,  0.4,  1.5, 0.1] \\
&\hat{P} = [[ 0.9,  0.1],
        [ 0.2,  0.8],
        [ 0.7,  0.3],
        [ 0.05 ,  0.95],
        [ 0.25,  0.75],
        [ 0.3,  0.7]]\\
\\
\text{(right)} \; \; &|\actions| = 2, \; \gamma = 0.9 \\
&\hat{r} = [-0.45, -0.1,  0.5,  0.5]\\
&\hat{P} = [[ 0.7,  0.3],
       [ 0.99,  0.01],
       [ 0.2,  0.8],
       [ 0.99,  0.01]]\\
\\
\text{In Section \ref{sec:dynamics}: }&|\actions| = 2, \; \gamma = 0.9 \\
&\hat{r} = [-0.45, -0.1,  0.5,  0.5]\\
&\hat{P} = [[ 0.7,  0.3],
       [ 0.99,  0.01],
       [ 0.2,  0.8],
       [ 0.99,  0.01]]
\end{align*}
\newpage

\section{Notation for the proofs}
In the section we present the notation that we use to establish the results in the main text.
The space of policies $\policies$ describes a Cartesian product of simplices that we can express as a space of $|\states| \times |\actions|$ matrices. However, we will adopt for policies, as well as the other components of $\mdp$, a convenient matrix form similar to \citep{wang2007dual}.
\begin{itemize}
\item The transition matrix $\transitions$ is a $|\states| |\actions| \times |\states| $ matrix denoting the probability of going to state $s'$ when taking action $a$ in state $s$ .
\item A policy $\pi$ is represented by a block diagonal $|\states| \times |\states| |\actions|$ matrix $M_{\pi}$. Suppose the state $s$ is indexed by $i$ and the action $a$ is indexed by $j$ in the matrix form, then we have that $M_{\pi}(i, i \times |\actions| + j) = \pi(a|s)$. The rest of the entries of $M_{\pi}$ are 0. From now on, we will confound $\pi$ and $M_{\pi}$ to enhance readability.
\item The transition matrix $\ppi=\pi \transitions$ induced by a policy $\pi$ is a $|\states| \times |\states|$ matrix denoting the probability of going from state $s$ to state $s'$ when following the policy $\pi$.
\item The reward vector $\rewards$ is a $|\states| |\actions| \times 1$ matrix denoting the expected reward when taking action $a$ in state $s$. The reward vector of a policy $r_\pi = \pi r$ is a $|\states| \times 1$ vector.
\item The value function $V^\pi$ of a policy $\pi$ is a $|\states| \times 1$ matrix.
\item We note $\col_i$ the $i$-th column of $(\eye - \gamma P^{\pi})^{-1}$.
\end{itemize}

Under these notations, we can define the Bellman operator $\bellop^{\pi}$ and the optimality Bellman operator $\bellop^*$ as follows:
\begin{align*}
&\bellop^{\pi} V^\pi = r_\pi + \gamma P^\pi V^\pi = \pi(r + \gamma P V^\pi) \\ 
\forall s \in \states, \; &\bellop^* V^{\pi}(s) = \max_{\pi' \in \policies} r_{\pi'}(s) + \gamma P^{\pi'} V^\pi(s).
\end{align*}

\section{Supplementary Results} \label{sec:supp_results}

\begin{lemma}
$f_v$ is infinitely differentiable on $\policies$.
\label{lm:differentiable}
\end{lemma}

\begin{proof}
We have that:
\begin{align*}
f_v(\pi) & = (\eye - \gamma \pi \transitions)^{-1}\pi\rewards \\
&= \frac{1}{\det(\eye - \gamma \pi \transitions)} \adj(\eye - \gamma \pi \transitions)\pi\rewards.
\end{align*}

Where $\det$ is the determinant and where $\adj$ is the adjunct. $\forall \pi \in \policies, \det(\eye - \gamma \pi \transitions) \neq 0$, therefore $f_v$ is infinitely differentiable.
\end{proof}

\begin{lemma}
Let $\pi \in \policies$, $s_1,..,s_k \in \states$, and $\pi' \in \agree$. We have
$$Span(\col_{k+1},..,\col_{|\states|}) =  Span(C^{\pi'}_{k+1},..,C^{\pi'}_{|\states|}).$$
\label{lm:matching}
\end{lemma}

\begin{proof}
As $P^\pi$ and $P^{\pi'}$ are equal on their first $k$ rows, we also have that $(\eye - \gamma P^\pi)$ and $(\eye - \gamma P^{\pi'})$ are equal on their first $k$ rows. We note these $k$ rows $L_1, ..., L_k$.

By assumption, we have that:
$$\forall i \in \{1,\dots,k\}, \forall j \in \{k+1,\dots,|\states|\}, L_i \col_{j}=0, L_i C^{\pi'}_{j}=0.$$

Which we can rewrite,
\begin{align*}
&Span(\col_{k+1},\dots,\col_{|\states|}) \subset Span(L_1,\dots, L_k)^{\bot} \\
&Span(C^{\pi'}_{k+1},..,C^{\pi'}_{|\states|}) \subset Span(L_1,\dots, L_k)^{\bot}
\end{align*}

Now using, $\dim Span(\col_{k+1},\dots,\col_{|\states|}) = \dim Span(C^{\pi'}_{k+1},..,C^{\pi'}_{|\states|}) = \dim Span(L_1,\dots, L_k)^{\bot} = |\states| - k$,  we have:
\begin{align*}
&Span(\col_{k+1},\dots,\col_{|\states|}) = Span(L_1,\dots, L_k)^{\bot} \\
&Span(C^{\pi'}_{k+1},..,C^{\pi'}_{|\states|}) = Span(L_1,\dots, L_k)^{\bot}.
\end{align*}

\end{proof}

\section{Proofs}

\lemmacompact*
\begin{proof}
$\policies$ is connected since it is a convex space, and it is compact because it is closed and bounded in a finite dimensional real vector space. Since $f_v$ is continuous (Lemma \ref{lm:differentiable}), we have $f_v(\policies) = \valuefunctions$ is compact and connected.
\end{proof}

\lmzeros*
\begin{proof}
Suppose without loss of generality that $\{s_1,..,s_k\}$ are the first $k$ states in the matrix form notation. We have,
\begin{align*}
\rewards_{\pi_1} &= \pi_1 \rewards\\
\rewards_{\pi_2} &= \pi_2 \rewards\\
\transitions^{\pi_1} &= \pi_1 \transitions\\
\transitions^{\pi_2} &= \pi_2 \transitions.
\end{align*}
Since $\pi_1(\dotbar s) = \pi_2(\dotbar s)$ for all $s \in \{s_1,..,s_k\}$, the first $k$ rows of $\pi_1, \pi_2$ are identical in the matrix form notation. Therefore, the first k elements of $\rewards_{\pi_1}$ and $\rewards_{\pi_2}$ are identical, and the first k rows of $\transitions^{\pi_1}$ and $\transitions^{\pi_2}$ are identical, hence the result. 
\end{proof}

\linelemmaone*
\begin{proof}
Let us first show that $f_v(\agree) \subset \valuefunctions \cap \affinesev$. \\
Let $\pi' \in \agree$, i.e. $\pi'$ agrees with $\pi$ on $s_1, .., s_k$. Using Bellman's equation, we have:
\begin{align}
V^{\pi'} - V^{\pi} &= r_{\pi'} - r_{\pi} + \gamma P^{\pi'} V^{\pi'} - \gamma P^{\pi} V^{\pi} \nonumber \\
&= r_{\pi'} - r_{\pi} + \gamma (P^{\pi'} - P^{\pi}) V^{\pi'} + \gamma P^{\pi}(V^{\pi'} - V^{\pi}) \nonumber \\
&= (\eye - \gamma P^{\pi})^{-1} \big( r_{\pi'} - r_{\pi} + \gamma (P^{\pi'} -P^{\pi})V^{\pi'} \big)
\label{eq:sev}.
\end{align}

Since the policies $\pi'$ and $\pi$ agree on the states $s_1,\dots,s_k$, we have, using Lemma \ref{lm:zeros}:
$$\left\{
    \begin{array}{ll}
        &r_{\pi'} - r_{\pi} \;\; \mbox{   is zero on its first k elements}\\
        &P^{\pi'} -P^{\pi} \mbox{ is zero on its first k rows.}
    \end{array}
\right.$$

Hence, the right-hand side of Eq.~\ref{eq:sev} is the product of a matrix with a vector whose first $k$ elements are 0. Therefore
\begin{align*}
    V^{\pi'} \in V^{\pi} + Span(\col_{k+1},..,\col_{|\states|})  \; .
\end{align*}

We shall now show that $\valuefunctions \cap \affinesev \subset f_v(\agree)$.

Suppose $V^{\hat{\pi}} \in \affinesev$. We want to show that there is a policy $\pi' \in \agree$ such that $V^{\pi'} = V^{\hat{\pi}}$. We construct $\pi'$ the following way:

$$\pi' = \left\{
    \begin{array}{ll}
        \pi(\dotbar s) & \mbox{if } s \in \{s_1,..,s_k\} \\
        \hat{\pi}(\dotbar s) & \mbox{otherwise.}
    \end{array}
\right.$$

Therefore, using the result of the first implication of this proof:
\begin{align*}
&V^{\hat{\pi}} - V^{\pi'} \in Span(\col_{k+1},..,\col_{|\states|}) \mbox{ by assumption} \\
&V^{\hat{\pi}} - V^{\pi'} \in Span(C^{\pi'}_1,..,C^{\pi'}_{k}) \mbox{ since } \hat{\pi} \mbox{ and }\pi' \mbox{ agree on } s_{k+1},\dots, s_{|\states|}. 
\end{align*}

However, as $\pi, \pi' \in \agree$, we have using Lemma \ref{lm:matching}:
\begin{align*}
&Span(\col_{k+1},..,\col_{|\states|}) =  Span(\colprime_{k+1},..,\colprime_{|\states|}).
\end{align*}

Therefore, $V^{\hat{\pi}} - V^{\pi'} \in Span(\colprime_{1},..,\colprime_{k}) \cap Span(\colprime_{k+1},..,\colprime_{|\states|}) = \{0\}$, meaning that $V^{\hat{\pi}} = V^{\pi'} \in  f_v(\agree)$.
\end{proof}

\lemmagproperties*
\begin{proof}

(i) $g$ is continuously differentiable as a composition of two continuously differentiable functions.\\

(ii) We want to show that we have either $V^{\pi_1} \lvec V^{\pi_0}$ or $V^{\pi_1} \gvec V^{\pi_0}$.

Suppose, without loss of generality, that $s$ is the first state in the matrix form. Using Lemma \ref{lm:freedom}, we have:

$$V^{\pi_0} = V^{\pi_1} + \alpha C_1^{\pi_1},\mbox{ with } \alpha \in \reals .$$

As $(\eye - \gamma P^{\pi_1})^{-1} = \sum_{i=0}^{\infty}(\gamma \pi \transitions)^{i}$, whose entries are all positive, $C^{\pi_1}_{1}$ is a vector with positive entries. Therefore we have $V^{\pi_1} \lvec V^{\pi_0}$ or $V^{\pi_1} \gvec V^{\pi_0}$, depending on the sign of $\alpha$.\\

(iii)  We have, using Equation (\ref{eq:sev})

\begin{align*}
V^{\pi_0} - V^{\pi_\mu} &= (\eye - \gamma P^{\pi_\mu})^{-1} \big( r_{\pi_0} - r_{\pi_\mu} + \gamma (P^{\pi_0} -P^{\pi_\mu})V^{\pi_0} \big)  \\
V^{\pi_0} - V^{\pi_1} &= (\eye - \gamma P^{\pi_1})^{-1} \big( r_{\pi_0} - r_{\pi_1} + \gamma (P^{\pi_0} -P^{\pi_1})V^{\pi_0} \big). 
\end{align*} 

Now, using
\begin{align*}
r_{\pi_0} &= \pi_0 r\\
r_{\pi_\mu} &= \pi_{\mu} r = \pi_{0}r + \mu (\pi_1 - \pi_0) r\\
P^{\pi_0} &= \pi_0 P\\
P^{\pi_\mu} &= \pi_{\mu} P = \pi_{0}P + \mu (\pi_1 - \pi_0) P \; ,
\end{align*}
we have
\begin{align}
V^{\pi_0} - V^{\pi_\mu} &= \mu (\eye - \gamma P^{\pi_\mu})^{-1} \big( r_{\pi_0} - r_{\pi_1} + \gamma (P^{\pi_0} -P^{\pi_1})V^{\pi_0} \big) \nonumber\\
&= \mu (\eye - \gamma P^{\pi_\mu})^{-1} (\eye - \gamma P^{\pi_1}) (V^{\pi_0} - V^{\pi_1}).
\label{eq:funcrat}
\end{align}

Therefore, $g(0)=g(1) \Rightarrow V^{\pi_0} - V^{\pi_1} = 0 \Rightarrow V^{\pi_0} - V^{\pi_\mu} = 0 \Rightarrow g(\mu) = g(0)$.

(iv) If $g(0) = g(1)$, the result is true since we can take $\rho = 0$ using (iii).

Suppose $g(0) \neq g(1)$, let us prove the existence of $\rho$ and that it is a rational function in $\mu$. Reusing the Equation \ref{eq:funcrat}, we have
\begin{align*}
V^{\pi_0} - V^{\pi_\mu} &= \mu (\eye - \gamma P^{\pi_\mu})^{-1} (\eye - \gamma P^{\pi_1}) (V^{\pi_0} - V^{\pi_1}) \\
&= \mu (\eye - \gamma(P^{\pi_{0}} + \mu(P^{\pi_1} - P^{\pi_0})) )^{-1} (\eye - \gamma P^{\pi_1}) (V^{\pi_0} - V^{\pi_1})\\
&= \mu (\eye - \gamma P^{\pi_1} - \gamma(1-\mu)(P^{\pi_{0}}-P^{\pi_{1}})) )^{-1} (\eye - \gamma P^{\pi_1}) (V^{\pi_0} - V^{\pi_1}).
\end{align*}

As we have that $P^{\pi_{0}}-P^{\pi_{1}}$ is a rank one matrix (Lemma \ref{lm:zeros}) that we can express as $P^{\pi_{0}}-P^{\pi_{1}}=uv^t$ with $u, v \in \reals^{|\states|}$. From the Sherman-Morrison formula:

\begin{align*}
(\eye - \gamma P^{\pi_1} - \gamma(1-\mu)(P^{\pi_{0}}-P^{\pi_{1}})) )^{-1} 
= (\eye - \gamma P^{\pi_1})^{-1} - \gamma(1-\mu)\frac{(\eye - \gamma P^{\pi_1})^{-1}(P^{\pi_{0}}-P^{\pi_{1}})(\eye - \gamma P^{\pi_1})^{-1}}{1 + \gamma(1-\mu)v^t(\eye - \gamma P^{\pi_1})^{-1}u}.
\end{align*}

Define $\omega_{\pi_1, \pi_0} = v^t(\eye - \gamma P^{\pi_1})^{-1}u$, we have

\begin{align*}
V^{\pi_0} - V^{\pi_\mu} 
= \mu V^{\pi_0} - \mu V^{\pi_1} - \frac{\gamma\mu(1-\mu)}{1 + \omega_{\pi_1, \pi_0}\gamma(1-\mu)} (\eye - \gamma P^{\pi_1})^{-1}(P^{\pi_{0}}-P^{\pi_{1}})(V^{\pi_0} - V^{\pi_1}).
\end{align*}

As in (i) we have that $(P^{\pi_{0}}-P^{\pi_{1}})(V^{\pi_0} - V^{\pi_1})$ is zeros on its last $\states - 1$ elements using an argument similar to Lemma \ref{lm:zeros}, therefore

$$(\eye - \gamma P^{\pi_1})^{-1}(P^{\pi_{0}}-P^{\pi_{1}})(V^{\pi_0} - V^{\pi_1}) = \beta_{\pi_0, \pi_1} C^{\pi_1}_1, \mbox{ with } \beta_{\pi_0, \pi_1} \in \reals.$$

Now recall from the proof of (i) that similarly, we have
$$V^{\pi_0} - V^{\pi_1} = \alpha_{\pi_0, \pi_1} C^{\pi_1}_1.$$

As by assumption $V^{\pi_0} - V^{\pi_1} \neq 0$, we have:
$$(\eye - \gamma P^{\pi_1})^{-1}(P^{\pi_{0}}-P^{\pi_{1}})(V^{\pi_0} - V^{\pi_1}) =  \frac{\beta_{\pi_0, \pi_1}}{\alpha_{\pi_0, \pi_1}}(V^{\pi_0} - V^{\pi_1}).$$

Finally we have,
\begin{align*}
V^{\pi_0} - V^{\pi_\mu} 
&= \mu V^{\pi_0} - \mu V^{\pi_1} - \frac{\gamma\mu(1-\mu)}{1 + \omega_{\pi_1, \pi_0}\gamma(1-\mu)}\frac{\beta_{\pi_0, \pi_1}}{\alpha_{\pi_0, \pi_1}}(V^{\pi_0} - V^{\pi_1}) \\
&= \Big(\mu - \frac{\gamma\mu(1-\mu)}{1 + \omega_{\pi_1, \pi_0}\gamma(1-\mu)}\frac{\beta_{\pi_0, \pi_1}}{\alpha_{\pi_0, \pi_1}}\Big)(V^{\pi_0} - V^{\pi_1}).
\end{align*}

Therefore, $\rho$ is a rational function in $\mu$, hence continuous, that we can express as:
$$\rho(\mu) = \mu - \frac{\gamma\mu(1-\mu)}{1 + \omega_{\pi_1, \pi_0}\gamma(1-\mu)}\frac{\beta_{\pi_0, \pi_1}}{\alpha_{\pi_0, \pi_1}}.$$

Now let us prove that $\rho$ is strictly monotonic. Suppose that $\rho$ is not strictly monotonic. As $\rho$ is continuous, we have that $\rho$ is not injective. Hence, $\exists \mu_0, \mu_1 \in [0, 1]$ distinct and the associated mixture of policies $\pi_{\mu_0}, \pi_{\mu_1}$ such that 

\begin{align*}
g(\mu_0) = g(\mu_1) & \Leftrightarrow V^{\pimuzero} = V^{\pimuone}\\
  &\Leftrightarrow \bellop^{\pimuzero} V^{\pimuzero} = \bellop^{\pimuone} V^{\pimuone}\\
  &\Leftrightarrow \bellop^{\pimuzero} V^{\pimuzero} = \bellop^{\pimuone} V^{\pimuzero} \\ 
  &\Leftrightarrow (\muzero \bellop^{\pizero} + (1-\muzero)\bellop^{\pione} ) V^{\pimuzero} =  (\muone \bellop^{\pizero} + (1-\muone)\bellop^{\pione} ) V^{\pimuzero}\\
  &\Leftrightarrow \muzero (\bellop^{\pizero} - \bellop^{\pione} ) V^{\pimuzero} = \muone (\bellop^{\pizero} - \bellop^{\pione} ) V^{\pimuzero}\\
  &\Leftrightarrow \bellop^{\pizero} V^{\pimuzero} = \bellop^{\pione} V^{\pimuzero}.
\end{align*}

Therefore we have
$$V^{\pimuzero} = \bellop^{\pimuzero} V^{\pimuzero} =  (\muzero \bellop^{\pizero} + (1-\muzero)\bellop^{\pione} ) V^{\pimuzero} = \bellop^{\pione} V^{\pimuzero}.$$

Therefore
$$\bellop^{\pizero} V^{\pimuzero} = \bellop^{\pione} V^{\pimuzero} = V^{\pimuzero}.$$

However, the Bellman operator has a unique fixed point, therefore
\begin{align*}
V^{\pizero} = V^{\pione} = V^{\pimuzero},
\end{align*}

which contradicts our assumption.
\end{proof}

\linetheorem*
\begin{proof}
Let us start by proving the first statement of the theorem which is the existence of two $s$-deterministic policies $\pi_u, \pi_l$ in $\agreeone$ that respectively dominates and is dominated by all other policies.\\ 

The existence of $\pi_l$ and $\pi_u$ (without enforcing their $s$-determinism), whose value functions are respectively dominated or dominate all other value functions of policies of $\agreeone$, is given by: 
\begin{itemize}
\item $f_v(\agreeone)$ is compact as an intersection of a compact and an affine plane (Lemma \ref{lm:freedom}).
\item There is a total order on this compact space ((ii) in Lemma \ref{lm:gproperties}).
\end{itemize}

Suppose $\pi_l$ is not $s$-deterministic, then there is $a \in \actions$ such that $\pi_l(a|s)=\mu^* \in (0,1)$. Hence we can write $\pi_l$ as a mixture of $\pi_1, \pi_2$ defined as follows

\begin{align}\forall s' \in \states, a' \in \actions,  \, &\pi_1(a'|s') = \left\{
  \begin{array}{ll}
    1 \mbox{ if } s'=s, a'=a \\
    0 \mbox{ if } s'=s, a' \neq a \\
    \pi_l(a'|s') \mbox{ otherwise. }
  \end{array}
\right.\\
&\pi_2(a'|s') = \left\{
  \begin{array}{ll}
    0 \mbox{ if } s'=s, a'=a \\
    \frac{1}{1-\mu^*}\pi_l(a'|s')   \mbox{ if } s'=s, a' \neq a \\
    \pi_l(a'|s') \mbox{ otherwise. }
  \end{array}
\right.
\end{align}

Therefore $\pi_l = \mu^* \pi_1 + (1-\mu^*) \pi_2$. We can use (iv) in Lemma \ref{lm:gproperties}, that gives that $g: \mu \mapsto f_v(\mu \pi_1 + (1-\mu) \pi_2)$ is either strictly monotonic or constant. If $g$ was strictly monotonic we would have a contradiction on $f_v(\pi_l)$ being minimum. Therefore $g$ is constant, and in particular 
$$f_v(\pi_l) = g(1) = f_v(\pi_1),$$

with $\pi_1 \; s$-deterministic. 

Similarly we can show there is an $s$-deterministic policy that has the same value function as $\pi_u$, hence proving the result.

Now let us prove the equivalence between (i), (ii) and (iii).
\begin{itemize}
\item Let $\pi' \in \agreeone$, we have: $f_v(\pi_l) \lvec f_v(\pi') \lvec f_v(\pi_u)$ and $f_v(\pi_l), f_v(\pi'), f_v(\pi_u)$ are on the same line (Lemma \ref{lm:freedom}). Therefore, $f_v(\pi')$ is a convex combination of $f_v(\pi_l)$ and $f_v(\pi_u)$ hence (i) $\subset$ (iii).
\item By definition, (ii) $\subset$ (i).
\item Lemma \ref{lm:gproperties} gives $\mu \mapsto f_v(\mu \pi_u + (1-\mu) \pi_l) = f_v(\pi_l) + \rho(\mu)(f_v(\pi_u) - f_v(\pi_l))$ with $\rho$ continuous and $\rho(0)=0, \rho(1)=1$. Using the theorem of intermediary values on $\rho$ we have that $\rho$ takes all values between $0$ and $1$. Therefore (iii) $\subset$ (ii)
\end{itemize}

We have (iii) $\subset$ (ii) $\subset$ (i) $\subset$ (iii). Therefore (i) = (ii) = (iii).

\end{proof} 

\thmconvexhull*
\begin{proof}
We prove the result by induction on the number of states $k$. If $k = 1$, the result is true by Theorem \ref{th:line}.

Suppose the result is true for $k$ states. Let $s_1,..,s_{k+1} \in \states$, we have by assumption that
$$\exists n \in \mathbb{N}, \pi_1,..,\pi_n \{s_1,..,s_k\}\mbox{-deterministic} \in \policies, \alpha_1,..,\alpha_n \in [0,1], \mbox{ s.t. }
\left\{
  \begin{array}{ll}
    V^{\pi} = \sum^{n}_{i=1}\alpha_i V^{\pi_i} \\
    \sum^{n}_{i=1}\alpha_i = 1
  \end{array}
\right.$$

However, using Theorem \ref{th:line}, we have

$\forall i \in [1,n], \exists \pi_{i,l}, \pi_{i,u} \in \policies,\left\{
  \begin{array}{lll}
    & \pi_{i,l}, \pi_{i,u} \; s_{k+1}\text{-deterministic}\\
    & \pi_{i,l}, \pi_{i,u} \; \text{agrees with }\pi_i \text{ on } s_1,..,s_k\\
    & \exists \beta_i \in [0, 1], V^{\pi_i} = \beta_i V^{\pi_{i,l}} + (1-\beta_i)V^{\pi_{i, u}} 
  \end{array}.
\right.$

Therefore
$$V^{\pi} = \sum^{n}_{i=1}\alpha_i (\beta_iV^{\pi_{i,l}} + (1-\beta_i)V^{\pi_{i, u}}),$$ thus concluding the proof.
\end{proof}

\polyhedrabyboundaries*
\begin{proof}
We will show the result by induction on the dimension of $K$.

For dim$(K)=1$, the proposition is true since $P$ is a polyhedron iff its boundary is a finite number of points.

Suppose the proposition is true for dim$(K)=n$, let us show that it is true for dim$(K)=n+1$. 

We can verify that if $P$ is a polyhedron, then:
\begin{itemize}
\item $P$ is closed.
\item There is a finite number of hyperplanes covering its boundaries (the boundaries of the half-spaces defining each convex polyhedron composing $P$).
\item The intersection of $P$ with these hyperplanes still are polyhedra.
\end{itemize}

Now let us consider the other direction of the implication, i.e. suppose that $P$ is closed, $\partial_K P \subset \cap_{i=1}^{k} H_i$, and $\forall i, P \cap H_i$ is a polyhedron. We will show that we can express $P$ as a finite union of polyhedra. \\

Suppose $x \in P$ and $x \notin \cup_{i=1}^{k} H_i$. By assumption, we have that $x \in \relint(P)$.
We will show that $x$ is in a intersection of closed half-spaces defined by the hyperplanes $H_1,..,H_k$ and that any other vector in this intersection is also in $P$ (otherwise we would have a contradiction on the boundary assumption).

A hyperplane $H_i$ defines two closed half-spaces denoted by $H_i^{+1}$ and $H_i^{-1}$ (the signs being arbitrary). And the intersections of those half-spaces form a partition of $K$, therefore:
$$ \exists \delta \in \{-1, 1\}^k, x \in \cap_{i=1}^{k} H_i^{\delta(i)} = P_{\delta}.$$
By assumption, $x \in \relint(\pdelta)$, since we assumed that $x \notin \cup_{i=1}^{k} H_i$. Now suppose $\exists y \in  \relint(\pdelta)$ s.t. $y \notin P$, we have: 
$$\exists \lambda \in [0,1], \lambda x + (1-\lambda)y = z \in \partial_K P.$$

However, $z \in \relint(\pdelta)$ because $\relint(\pdelta)$ is convex. Therefore $z \notin \bigcup^{k}_{i=1} H_i$ since $z \in \relint(\pdelta)$ which gives a contradiction. We thus have either $\relint(\pdelta) \subset P$ or $\relint(\pdelta) \cap P = \emptyset$.

Now suppose $\relint(\pdelta) \subset P$ and $\relint(\pdelta)$ nonempty. We have that $\cl(\relint(\pdelta)) = \pdelta$ \citep[Theorem~3.3]{brondsted2012introduction} and $P$ closed, meaning that $\pdelta \subset P$. Therefore, we have

$$\exists \delta_1,..,\delta_{j} \in \{-1, 1\}^k \text{ s.t. } P = (\cup^{k}_{i=1} P \cap H_i) \bigcup (\cup^{j}_{i=1} P_{\delta_i}).$$

$P$ is thus a finite union polyhedra, as $\{P \cap H_i\}$ are polyhedra by assumption, and $\{P_{\delta_i}\}$ are convex polyhedra by definition.
\end{proof}

\nconnect*
\begin{proof} 
We can define the policies $\pi_2,..,\pi_{|\states|-1}$ the following way:

$$\forall i \in [2, |\states|-1], \left\{
  \begin{array}{ll}
    \pi_i(\dotbar s_j) = \pi'(\dotbar s_j) &\text{ if } s_j \in \{s_1,..,s_{i-1}\}\\
    \pi_i(\dotbar s_j) = \pi(\dotbar s_j) &\text{ if } s_j \in \{s_i,..,s_{|\states|}\}
  \end{array}
\right.$$

Therefore, two consecutive policies $\pi_i, \pi_{i+1}$ only differ on one state. We can apply Theorem \ref{th:line} and thus conclude the proof.
\end{proof}

\thneighborhood*
\begin{proof}
We will prove the result by showing that $V^{\pi}$ is in $|\states| - k$ segments that are linearly independent by applying the line theorem on a policy $\hat{\pi}$ that has the same value function as $\pi$. We will then be able to conclude using the regularity of $f_v$.    

We can find a policy $\hat{\pi} \in \agree$, that has the same value function as $\pi$ by applying recursively Theorem \ref{th:line} on the states $\{s_{k+1},..,s_{|\states|}\}$, such that:\\

\(\exists a_{k+1,l}, a_{k+1, u},.., a_{|\states|,l}, a_{|\states|, u} \in \actions, \forall i \in \{k+1,..,|\states|\},\)
\[\hat{\pi}(a_{i,l}|s_i) = 1 - \hat{\pi}(a_{i,u}|s_i) = \hat{\mu}_i \in \; (0,1). \]

Note that $\hat{\mu}_i \notin \; \{0,1\}$ because we assumed that no $s$-deterministic policy has the same value function as $\pi$.\\
We define $\hat{\mu} = (\hat{\mu}_{k+1},...,\hat{\mu}_{|\states|}) \in (0, 1)^{|\states|-k}$  and the function $g:\; (0, 1)^{|\states|-k} \rightarrow \affinesev$ such that:
$$g(\mu) = f_v(\pi_\mu), \text{ with } \left\{
  \begin{array}{ll}
    \pi_\mu(a_{i,l}|s_i) = 1- \pi_\mu(a_{i,u}|s_i) = \mu_i &\text{ if } i \in \{k+1,\dots,|\states|\}\\
    \pi_\mu(\dotbar s_i) = \hat{\pi}(\dotbar s_i) &\text{ otherwise.}
  \end{array}
\right.$$

We have that:
\begin{enumerate}
\item g is continuously differentiable
\item $g(\hat{\mu}) = f_v(\hat{\pi})$
\item $\frac{\partial g}{\partial \mu_i}$ is non-zero at $\hat{\mu}$ (Lemma \ref{lm:gproperties}.iv)
\item $\frac{\partial g}{\partial \mu_i}$ is along the $i$-th column of $(I - \gamma P^{\hat{\pi}})^{-1}$ (Lemma \ref{lm:freedom})
\end{enumerate}

Therefore, the Jacobian of $g$ is invertible at $\mu$ since the columns of  $(I - \gamma P^{\hat{\pi}})^{-1}$ are independent, therefore by the inverse theorem function, there is a neighborhood of $g(\mu) = f_v(\hat{\pi})$ in the image space, which gives the result.
\end{proof}

\corolboundary*
\begin{proof}
Let $V^\pi \in \valuefunctions^y$; from Theorem \ref{th:neighbor}, we have that

$$ V^\pi \notin  \bigcup_{i=k+1}^{|\states|} \bigcup_{j=1}^{|\actions|} f_v(\agree \cap D_{s_i, a_j}) \Rightarrow V^\pi \in \text{relint}(V^y)
$$
where $\text{relint}$ refers to the relative interior in $\affinesev$.

Therefore, 
$$\partial \valuefunctions^y \subset \bigcup_{i=k+1}^{|\states|} \bigcup_{j=1}^{|\actions|} f_v(\agree \cap D_{s_i, a_j}).$$
\end{proof}

\thboundaries*
\begin{proof}
We prove the result by induction on the cardinality of the number of states $k$. \\
If $k=|\states|$, then $\agree = \{f_v(\pi)\}$ which is a polytope.\\

Suppose that the result is true for $k+1$, let us show that it is still true for $k$.

Let $\pi \in \Pi, s_1, .., s_k \in \states$, 
define the ensemble $\agree$ of policies that agree with $\pi$ on $\{s_1,..,s_k\}$ and $\valuefunctions^y = f_v(\agree)$. Using Lemma \ref{lm:freedom}, we have that $\valuefunctions^y = \valuefunctions \cap \affinesev$. \\

Now, using Corollary \ref{cr:boundary}, we have that:

$$\partial \valuefunctions^y \subset \bigcup_{i=k+1}^{|\states|} \bigcup_{j=1}^{|\actions|} f_v(\agree \cap D_{s_i, a_j}) = \bigcup_{i=k+1}^{|\states|} \bigcup_{j=1}^{|\actions|} \valuefunctions^y \cap H_{i,j},$$where $\partial$ refer to the relative boundary in $\affinesev$, and $H_{i,j}$ is an affine hyperplane of $\affinesev$ (Lemma \ref{lm:freedom}).

Therefore we have that:
\begin{enumerate}
\item $\valuefunctions^y = \valuefunctions \cap \affinesev$ is closed since it is an intersection of two closed ensembles
\item $\partial \valuefunctions^y \subset \bigcup_{i=k+1}^{|\states|} \bigcup_{j=1}^{|\actions|} H_{i,j}$ affine hyperplanes in $\affinesev$
\item $\valuefunctions^y \cap H_{i,j} = f_v(\agree \cap D_{s_i, a_j})$ is a polyhedron (induction assumption).
\end{enumerate}

$\valuefunctions^y$ verifies (i), (ii), (iii) in Proposition \ref{prop:tope}, therefore it is a polyhedron. We have $\valuefunctions^y$ bounded since $\valuefunctions^y \subset \valuefunctions$ bounded. Therefore, $\valuefunctions^y$ is a polytope.
\end{proof}



\end{document}